\documentclass[runningheads]{llncs}
\usepackage[T1]{fontenc}
\usepackage{graphicx}
\usepackage{times}
\usepackage{soul}
\usepackage{url}
\usepackage[hidelinks]{hyperref}
\usepackage[utf8]{inputenc}
\usepackage[small]{caption}
\usepackage{graphicx}
\usepackage{amsmath}
\usepackage{booktabs}
\usepackage{algorithm}
\usepackage{algorithmic}
\usepackage[british]{babel}
\usepackage[usenames,x11names]{xcolor}
\usepackage{relsize,setspace}
\usepackage{stmaryrd}
\usepackage[font=it]{caption}
\usepackage{tikz,pgf}\usetikzlibrary{shapes,backgrounds}

\usepackage{amssymb}
\usepackage{pifont}

\usepackage{cleveref}  

\usepackage{multirow}
\usepackage{comment}
\usepackage{xspace}
\usepackage{pdflscape}
\usepackage{yhmath}

\usepackage[inline]{enumitem}

\setlist{leftmargin=*,noitemsep,parsep=0.2ex,topsep=0.3ex}
\usepackage{scalerel}

\usepackage{microtype}

\usepackage{ifthen}
\newif\iflong
\longfalse

\usepackage{todonotes}

\newcommand{\pred}[1]{\small\mbox{\tt #1}}

\newcommand{\spform}[1]{\mathsf{#1}}

\newcommand{\mathcom}[3]{ \newcommand{#1}[#2]{\mbox{$#3$}}}
\mathcom{\imp}{0}{\ \rightarrow\ }            
\mathcom{\rimp}{0}{\ \leftarrow\ }            

\mathcom{\con}{0}{\ \wedge\ }                 
\mathcom{\dis}{0}{\ \vee\ }                   
\mathcom{\n}{0}{\neg}                     
\mathcom{\dimp}{0}{\ \leftrightarrow\ }       
\mathcom{\corresponds}{0}{\ \Lleftarrow\! \! \Rrightarrow\ }
\mathcom{\A}{0}{\forall}                  
\mathcom{\E}{0}{\exists}     

\def\Box{\mathop\square}
\def\Diamond{\mathop\lozenge}

\mathcom{\tuple}{1}{\langle #1 \rangle}

\def\eqdef{\mathrel{\ =_{\mbox{\em \tiny def}}\ }}

\def\textquote#1{``#1''}

\def\PredStandpointLogic{{\mathbb{S}_{\scaleobj{0.85}{\mathrm{FO}}}}}
\def\SSFO{{\mathbb{S}_{\scaleobj{0.85}{\mathrm{[FO]}}}}}
\def\StandpointSROIQ{{\mathbb{S}_{\scaleobj{0.85}{[\mathcal{S\!R\!O\!I\!Q}b_s}\!]}}}
\def\StandpointSROIQ{{\mathbb{S}_{\scaleobj{0.85}{[\mathcal{S\!R\!O\!I\!Q}b_s}\!]}}}
\def\SSS{\StandpointSROIQ}
\def\VSSS{{\mathbb{V}_{\scaleobj{0.85}{[\mathcal{S\!R\!O\!I\!Q}b_s}\!]}}}

\def\FOformulas{\PredStandpointLogic}

\def\st{\hbox{$\spform{s}$}\xspace}

\def\star{\hbox{$*$}\xspace}

\def\pr{\pi}

\def\Precs{\Pi}

\def\Precsf#1{\Precs_{#1}}
\def\Precsmf#1#2{\Precs_{#1,#2}}

\def\standb#1{\Box\nolimits_{\spform{#1}}}
\def\standd#1{\Diamond\nolimits_{\spform{#1}}}

\def\standbx#1{\Box\nolimits_{\scaleobj{0.8}{\spform{#1}}}}
\def\standdx#1{\Diamond\nolimits_{\scaleobj{0.8}{\spform{#1}}}}

\def\standbe{\standb{e}}

\def\standde{\standd{e}}

\def\allstandb{\standb{*}}
\def\allstandd{\standd{*}}

\def\standindef#1{\mathcal{I}_{\scaleobj{0.8}{\mathsf{#1}}}}
\def\allstandindef{\mathcal{I}_{*}}
\def\standdef#1{\mathcal{D}_{\scaleobj{0.8}{\mathsf{#1}}}}

\def\model{\mathfrak{M}}

\xspace
\def\f\xspacestandtopre{\hbox{$\sigma\,$}\xspace}
\def\fpretov\xspacealue{\hbox{$\delta\,$}\xspace}

\def\ModSat#1||-#2{#1\models #2}
\def\NotModSat#1||-#2{#1\nvDash #2}

\newcommand{\skipit}[1]{} 
\newcommand{\addit}[1]{} 

\newcommand{\define}[1]{\emph{#1}}
\newcommand{\N}{\mathbb{N}}
\newcommand{\set}[1]{\left\{#1\right\}}
\newcommand{\guard}{\mathrel{\left.\middle\vert\right.}}

\newcommand{\Q}{\rotatebox[origin=c]{180}{\scalebox{0.90}{\raisebox{-1.6ex}[0ex][0ex]{$\mathsf{Q}$}}}}
\renewcommand{\eqdef}{\mathrel{\,:=\,}}

\newcommand{\ebnfeq}{\mathrel{::=}}
\newcommand{\bigland}{\bigwedge}

\newcommand{\limplies}{\rightarrow}
\renewcommand{\land}{\mathrel{\wedge}}
\renewcommand{\lor}{\mathrel{\vee}}
\newcommand{\Trans}{\mathrm{Trans}}
\newcommand{\trans}{\mathrm{trans}}
\newcommand{\svec}[1]{\vec{#1}\,}

\newcommand{\Sub}{\mathit{Sub}}

\def\Stands{\mathcal{S}}
\def\sts{\spform{s}} 
\def\Preds{\mathcal{P}}

\def\Consts{\mathcal{C}}
\def\Vars{\mathcal{V}}
\def\Terms{\mathcal{T}}
\def\E{\mathcal{E}}
\def\StandExps{\E_{\Stands}}
\def\ste{\spform{e}} 
\def\sigmaE{\sigma_{\E}}
\def\transE{\trans_{\E}}

\def\kstruct{\mathfrak{M}}

\def\struct{\mathcal{I}}
\def\varassign{v}
\def\intf{\cdot^{\struct}}
\def\Dom{\Delta}
\def\de{\delta}
\newcommand{\interpret}[2]{#1^{#2}}
\newcommand{\interprets}[1]{\interpret{#1}{\struct}}
\newcommand{\interpretsv}[1]{\interpret{#1}{\struct,\varassign}}
\newcommand{\interpretgp}[1]{\interpret{#1}{\gamma(\pr)}}
\newcommand{\interpretgpv}[1]{\interpret{#1}{\gamma(\pr),\varassign}}

\def\dland{\sqcap}
\def\dlor{\sqcup}
\def\dlsub{\sqsubseteq}

\def\FOSL{FOSL\xspace}
\def\fosl{first-or\-der stand\-point logic\xspace}

\def\foss{first-or\-der stand\-point structure\xspace}
\def\fosss{first-or\-der stand\-point structures\xspace}

\newcommand{\tad}{\hspace{1pt}}

\newenvironment{proof}{\paragraph{Proof.}}{\hfill\qed}
\newenvironment{claimproof}{\paragraph{Proof (of the claim):}}{\hfill\textit{(This proves the claim.)}}

\newenvironment{narrowalign}{\\[3pt]\mbox{}\hfill}{\hfill\mbox{}\\[4pt]}

\newcommand{\SROIQbs}{\ensuremath{\mathcal{SROIQ}b_s}\xspace}

\newcommand{\Inter}{\mathcal{I}} 

\RequirePackage{graphicx}

\newcommand{\atleast}[1]{\mathord{\geqslant}#1\,}
\newcommand{\atmost}[1]{\mathord{\leqslant}#1\,}

\newcommand{\conc}[1]{#1}
\newcommand{\rol}[1]{#1}

\newcommand{\rolexpR}{\rol{r}}

\newcommand{\rolR}{\mathtt{r}}
\newcommand{\rolS}{\mathtt{s}}
\newcommand{\rolT}{\mathtt{t}}
\newcommand{\rolU}{\mathtt{u}}

\newcommand{\conA}{\mathtt{A}}
\newcommand{\conB}{\mathtt{B}}
\newcommand{\conC}{\conc{C}}
\newcommand{\conD}{\conc{D}}

\newcommand{\conS}{\mathtt{M}}

\newcommand{\sroiq}{\text{$\mathcal{SROIQ}$}\xspace}

\newcommand{\Self}{\ensuremath{\mathit{Self}}}

\newif\ifsupplementary
\supplementaryfalse
\supplementarytrue
\newif\ifdl
\dlfalse
\dltrue

\ifsupplementary\renewcommand{\skipit}[1]{\color{teal}#1\color{black}}\fi
\ifdl\renewcommand{\skipit}[1]{#1}\fi

\begin{document}
\title{How to Agree to Disagree}
\subtitle{Managing Ontological Perspectives using Standpoint Logic\ifsupplementary\ifdl\else\\[1ex]\textcolor{teal}{(Version with Supplementary Material)}\fi\fi}%
%
\author{Lucía Gómez Álvarez\orcidID{0000-0002-2525-8839} \and Sebastian Rudolph\orcidID{0000-0002-1609-2080} \and Hannes Strass\orcidID{0000-0001-6180-6452}}
%
\authorrunning{L.\ Gómez Álvarez et al.}
%
\institute{Computational Logic Group, Faculty of Computer Science, TU Dresden, Germany\\
	\email{\{lucia.gomez\_alvarez,\,sebastian.rudolph,\,hannes.strass\}@tu-dresden.de}}
\maketitle              
\begin{abstract}
	The importance of taking individual, potentially conflicting perspectives into account when dealing with knowledge has been widely recognised.
	Many existing ontology management approaches fully merge knowledge perspectives, which may require weakening in order to maintain consistency; others represent the distinct views in an entirely detached way.

	As an alternative, we propose \emph{Standpoint Logic},
	a simple, yet versatile multi-modal logic ``add-on'' for existing KR languages 
	intended for the integrated representation of domain knowledge relative to diverse, possibly conflicting \emph{standpoints}, which can be hierarchically organised, combined, and put in relation with each other.

	Starting from the generic framework of \emph{First-Order Standpoint Logic} (FOSL), we subsequently focus our attention on the fragment of \emph{sentential} formulas, for which we provide a polytime translation into the standpoint-free version. This result yields decidability and favourable complexities for a variety of highly expressive decidable fragments of first-order logic. Using some elaborate encoding tricks, we then establish a similar translation for the very expressive description logic $\mathcal{SROIQ}b_s$ underlying the OWL 2 DL ontology language. By virtue of this result, existing highly optimised OWL reasoners can be used to provide practical reasoning support for ontology languages extended by standpoint modelling.

\end{abstract}
%



%
%
\section{Introduction}\label{sec:introduction}

Artefacts of contemporary knowledge representation (ontologies, knowledge bases, or knowledge graphs) serve as means to conceptualise specific domains, with varying degrees of expressivity ranging from simple classifications and database schemas to fully axiomatised theories.
Inevitably, such specifications reflect the individual points of view of their creators (be it on a personal or an institutional level) along with other contextual aspects, and they may also differ in modelling design decisions, such as the choice of conceptual granularity or specific ways of axiomatising information.
This semantic heterogeneity is bound to pose significant challenges whenever the interoperability of independently developed knowledge specifications is required.

This paper proposes a way to address the interoperability challenge while at the same time preserving the varying perspectives of the original sources. This is particularly important in scenarios that require the simultaneous consideration of multiple, potentially contradictory, viewpoints.

\begin{example}\label{example:fol}
	A broad range of conceptualisations and definitions for the notion of \emph{forest} have been specified for different purposes, giving rise to diverging or even contradictory statements regarding forest distributions.
	Consider a knowledge integration scenario involving two sources adopting a \emph{land cover} $(\mathsf{LC})$ and a \emph{land use} $(\mathsf{LU})$ perspective on forestry.
	$\mathsf{LC}$ characterises a \emph{forest} as a ``forest ecosystem'' with a minimum area (\ref{formula:defForestlandCover_sroiq}) where a \emph{forest ecosystem} is specified as an ecosystem with a certain ratio of tree canopy cover (\ref{formula:defForestEcosystem_sroiq}).
	$\mathsf{LU}$ defines a forest with regard to the purpose for which an area of land is put to use by humans, i.e.\ a forest is a maximally connected area with ``forest use'' (\ref{formula:defForestlandUse_sroiq}).\footnote{``Forest use'' areas may qualify for logging and mining concessions as well as be further classified into, e.g.\ agricultural or recreational land use.}

	Both sources $\mathsf{LC}$ and $\mathsf{LU}$ agree that forests subsume broadleaf, needleleaf and tropical forests (\ref{formula:forestSubclasses}), and they both adhere to the Basic Formal Ontology ($\mathsf{BFO}$, \cite{arp2015building}), an upper-level ontology that formalises general terms, stipulating for instance that \emph{land} and \emph{ecosystem} are disjoint categories {\rm(\ref{formula:disjointLandEco})}. 
	Using standard description logic notation and providing ``perspective annotations'' by means of correspondingly labelled box operators borrowed from multi-modal logic, the above setting might be formalised as follows:

	\begin{enumerate}[series=SroiqForestry,label={\rm (F\arabic*)},ref={\rm F\arabic*},leftmargin=2.7em,labelwidth=2.4em]\small
		\item $\standbx{LC}[\pred{Forest} \equiv \pred{ForestEcosystem}\dland \exists \pred{hasLand}. \pred{Area}_{\geq 0.5\mathrm{ha}}]$\label{formula:defForestlandCover_sroiq}
		\item $\standbx{LC}[\pred{ForestEcosystem}\equiv  \pred{Ecosystem} \dland \pred{TreeCanopy}_{\geq 20\%}]$\label{formula:defForestEcosystem_sroiq}
		\item $\standbx{LU}[\pred{Forest}\equiv\pred{ForestlandUse}\dland \pred{MCON}] \con \allstandb[ \pred{ForestlandUse}\sqsubseteq \pred{Land}]$\label{formula:defForestlandUse_sroiq}
		\item $\standbx{LC\cup LU}[(\pred{BroadleafForest}\dlor \pred{NeedleleafForest}\dlor \pred{TropicalForest}) \dlsub \pred{Forest}]$\label{formula:forestSubclasses}
		\item $(\mathsf{LC}\preceq\mathsf{BFO}) \con (\mathsf{LU}\preceq\mathsf{BFO}) \con \standbx{BFO}[\pred{Land}\dland\pred{Ecosystem}\dlsub\bot]$ \label{formula:disjointLandEco}

	\end{enumerate}\vspace{-1ex}

\end{example}

In the case of \Cref{example:fol}, \emph{ecosystem} and \emph{land} are disjoint categories according to the overarching $\mathsf{BFO}$ (\ref{formula:disjointLandEco}), yet forests are defined as ecosystems according to $\mathsf{LC}$ (\ref{formula:defForestlandCover_sroiq}) and as lands according to ${\mathsf{LU}}$ (\ref{formula:defForestlandUse_sroiq}). These kinds of disagreements result in well-reported challenges in the area of Ontology Integration  \cite{Euzenat2008OntologyAlignments,Otero-Cerdeira2015OntologyReview} and make ontology merging a non-trivial task, often involving a certain knowledge loss or weakening in order to avoid incoherence and inconsistency \cite{Pesquita2013ToAlignments,Solimando2017MinimizingEvaluation}.
In \Cref{example:fol}, to merge $\mathsf{LU}$ and $\mathsf{LC}$, there are two typical options to resolve the issue: (Opt-Weak) one may give up on the disjointness axiom (\ref{formula:disjointLandEco}), or (Opt-Dup) one could duplicate all the conflicting predicates \cite{Osman2021OntologyIssues}, in this case not only $\pred{Forest}$ (into $\pred{Forest\_LC}$ and $\pred{Forest\_LU}$), but also the forest subclasses in (\ref{formula:forestSubclasses}): $\pred{BroadleafForest}$, $\pred{NeedleleafForest}$ and $\pred{TropicalForest}$.
In contrast, we advocate a multi-perspective approach that can represent and reason with many -- possibly conflicting -- standpoints, instead of focusing on combining and merging different sources into a single conflict-free conceptual model.

\emph{Standpoint logic} \cite{gomez2021standpoint} is a formalism inspired by the theory of supervaluationism \cite{Fine1975} and rooted in modal logic that supports the coexistence of multiple standpoints and the establishment of alignments between them, by extending the base language with labelled modal operators.
Propositions  $\standbx{LC}\phi$ and $\standdx{LC}\phi$ express information relative to the \emph{standpoint} $\mathsf{LC}$ and read, respectively: ``according to $\mathsf{LC}$, it is \emph{unequivocal/conceivable} that $\phi$''.
In the semantics, standpoints are represented by sets of \emph{precisifications},\footnote{Precisifications are analogous to the \emph{worlds} of frameworks with possible-worlds semantics.} such that $\standbx{LC}\phi$ and $\standdx{LC}\phi$ hold if $\phi$ is true in all/some of the precisifications in $\mathsf{LC}$.

The logical statements {(\ref{formula:defForestlandCover_sroiq})--(\ref{formula:disjointLandEco})}, which formalise \Cref{example:fol} by means of a stand\-point-enhanced description logic, are not inconsistent, so all axioms can be jointly represented. Let us now illustrate the use of standpoint logic for reasoning with the individual perspectives. First, assume the following (globally agreed) facts about three instances, an ecosystem $e$, a parcel of land $l$, and a city $c$:

\begin{enumerate}[resume=SroiqForestry,label={\rm (F\arabic*)},ref={\rm F\arabic*}]\small
	\item $\pred{ForestEcosystem}(e) \qquad \pred{hasLand}(e,l) \qquad \pred{ForestlandUse}(l)$\label{formula:instanceForestEco_sroiq}
	\item $\pred{Area}_{\geq 0.5\mathrm{ha}}(l) \qquad \pred{MCON}(l) \qquad \pred{in}(l,\mathit{c}) \qquad \pred{City}(\mathit{c})$\label{formula:instanceInDresden_sroiq}\label{formula:instanceForestArea_sroiq}
\end{enumerate}

It is clear from (\ref{formula:defForestlandCover_sroiq}) that according to $\mathsf{LC}$, $e$ is a forest, written as  $\standbx{LC}[\pred{Forest}(e)]$, since it is a forest ecosystem (\ref{formula:instanceForestEco_sroiq}) with an area larger than $0.5\mathrm{ha}$ (\ref{formula:instanceForestArea_sroiq}).
On the other hand, it is clear from (\ref{formula:defForestlandUse_sroiq}) that according to $\mathsf{LU}$, $l$ is a forest, $\standbx{LU}[\pred{Forest}(l)]$, since it has a forest land use (\ref{formula:instanceForestEco_sroiq}) and it is a maximally connected area (\ref{formula:instanceForestArea_sroiq}).
More interestingly, we can also obtain \emph{joint inferences}: assuming the (generally accepted) background knowledge expressed by $ \pred{hasLand}\circ\pred{in}\dlsub\pred{in}$, we can infer
$$\standbx{LC\cup LU}[ (\pred{City} \dland \exists\pred{in}^{-}.\pred{Forest})(\mathit{c})],$$
which means that ``according to both $\mathsf{LC}$ and $\mathsf{LU}$ there is some forest in City $c$.''
This holds for $\mathsf{LU}$ since $l$ is a forest and is in $c$ (\ref{formula:instanceInDresden_sroiq}); and
it holds for $\mathsf{LC}$ because $e$ is a forest in the land $l$, which is in turn in $c$.

In contrast to the options of the ontology merging approach, using standpoint logic prevents the multiplication (and corresponding ``semantic detachment'') of predicates from (Opt-Dup). It also avoids unintended consequences arising when knowledge sources are weakened just enough to maintain satisfiability: In the corresponding (Opt-Weak) scenario, after merging the knowlege sources of $\mathsf{LU}$ and $\mathsf{LC}$ and removing~(\ref{formula:disjointLandEco}), we can consistently infer $\pred{Forest}(e)$ from the standpoint-free versions of (\ref{formula:defForestlandCover_sroiq}), (\ref{formula:instanceForestEco_sroiq}) and (\ref{formula:instanceForestArea_sroiq}) and $\pred{Forest}(l)$ from (\ref{formula:defForestlandUse_sroiq}), (\ref{formula:instanceForestEco_sroiq}) and (\ref{formula:instanceForestArea_sroiq}), similar to the standpoint framework. But on top of that, reapplying (\ref{formula:defForestlandCover_sroiq}) and (\ref{formula:defForestlandUse_sroiq}) also yields ``$e$ is a forest, and its land $l$ is also a forest and an ecosystem, and has some other associated land, bigger than $0.5\mathrm{ha}$'' through the following derivable assertions:
$$\pred{Forest}(e) \ \ \ \pred{hasLand}(e,l) \ \ \  \pred{Forest}(l) \ \ \  \pred{ForestEcosystem}(l) \ \ \  \exists\pred{hasLand}.\pred{Area}_{\geq0.5\mathrm{ha}}(l)$$

This illustrates how, beyond the problem of inconsistency, naively merging different models of a domain may lead to erroneous reasoning. In fact, 
other non-clashing differences between the forest definitions (\ref{formula:defForestlandCover_sroiq}) and (\ref{formula:defForestlandUse_sroiq}) respond to relevant nuances that relate to each standpoint and should also not be naively merged.
For instance, from the land cover perspective (\ref{formula:defForestlandCover_sroiq}), there is no spatial connectedness requirement, since there are ``mosaic forest ecosystems'' where the landscape displays forest patches that are sufficiently close to constitute a single ecosystem.
On the other hand, for $\mathsf{LU}$, there is no minimum tree canopy (\ref{formula:defForestlandUse_sroiq}), since a temporarily cleared area still has a ``forest use''.

Standpoint logic preserves the independence of the perspectives and escapes global inconsistency -- without weakening the sources or duplicating entities -- because its model theory (cf.~\Cref{sec:semantics}) requires consistency only within standpoints and precisifications. Notwithstanding, it allows for the specification of structures of standpoints and alignments between them. Natural reasoning tasks over such multi-standpoint specifications include gathering unequivocal or undisputed knowledge, determining knowledge that is relative to a standpoint or a set of them, and contrasting the knowledge that can be inferred from different standpoints.

\medskip

Let us get an idea of the expressivity of the proposed logic.
In spite of its simple syntax, the language is remarkably versatile; it allows for specifying knowledge relative  (a) to a standpoint, e.g.\ (\ref{formula:defForestlandCover_sroiq}), (b) to the global standpoint, denoted by $\star$, e.g.\ (\ref{formula:defForestlandUse_sroiq}), and (c) to set-theoretic combinations of standpoints, e.g.\ (\ref{formula:forestSubclasses}).
Additional language features can be defined in terms of the former: $\standindef{\mathsf{LC}}\phi$, which means that, ``according to $\mathsf{LC}$, it is inherently \emph{indeterminate} whether $\phi$'' can be defined by \mbox{$\standindef{\mathsf{LC}}\phi \eqdef \standdx{LC}\mkern-1mu\phi \con \standdx{LC}\mkern-2mu\neg\phi$}. The \emph{sharper} operator $\preceq$ is used to establish hierarchies of standpoints and constraints on the structure of precisifications, e.g.\ (\ref{formula:disjointLandEco}), and can be defined via \mbox{$\st_1 \preceq \st_2 \eqdef \standb{\st_1\backslash\st_2}[\top \sqsubseteq \bot]$}. Intuitively, \mbox{$\st_1 \preceq \st_2$} expresses that standpoint $\st_1$ inherits the propositions of $\st_2$, by virtue of ``\mbox{$\st_1 \subseteq \st_2$}'' holding for the corresponding sets of precisifications. This type of statement comes handy to ``import'' background knowledge from some ontology, such as the foundational ontology BFO in our example.
In combination, these modelling features allow for expressing further constraints useful for knowledge integration scenarios, e.g.,
\begin{enumerate}[resume=SroiqForestry,leftmargin=2.7em,label={\rm (F\arabic*)},ref={\rm F\arabic*}]
	\small
	\item \ \
	      $\mathsf{*} \preceq (\mathsf{LC\cup LU} ) \ \con \ \standdx{LC}\mkern-1mu[\top{\,\sqsubseteq\,} \top]  \ \wedge \ \standdx{LU}\mkern-1mu[\top{\,\sqsubseteq\,} \top], $\label{formula:only-union}
\end{enumerate}
where the first conjunct allows us to specify that no interpretations beyond the standpoints of interest are under consideration, by stating that the universal standpoint is a subset of the union of $\mathsf{LC}$ and $\mathsf{LU}$. The other two conjuncts enforce the non-emptiness of the standpoints of interest, $\mathsf{LC}$ and $\mathsf{LU}$, ensuring that each standpoint by itself is coherent.
To illustrate a use case, consider the statement $\allstandd[\pred{Forest}(f)\con\n\pred{MCON}(f)]$, expressing that it is conceivable that $f$ is a non-spatially-connected forest. From this, we can infer together with (\ref{formula:only-union}) and the unfulfilled requirement of connectedness of $\mathsf{LU}$ (\ref{formula:defForestlandUse_sroiq}), that $f$ must be conceivable for $\mathsf{LC}$ instead, and thus $f$ must be a forest ecosystem (\ref{formula:defForestlandCover_sroiq}):
\begin{narrowalign}
	$\standdx{\mathsf{LC}}[\pred{ForestEcosystem}(f) \con (\exists\hspace{1pt}\pred{hasLand}.\pred{Area}_{\geq0.5\mathrm{ha}})(f)]$
\end{narrowalign}

\vspace*{-1em}

Gómez Álvarez and Rudolph~\cite{gomez2021standpoint} have introduced the standpoint framework  over a propositional base logic.
While they showed favourable complexity results (standard reasoning tasks are \textsc{NP}-complete just like for plain propositional logic), 
the framework is not expressive enough for knowledge integration scenarios employing contemporary ontology languages.
In this paper, we widen the scope by (1) introducing the very general framework of \emph{first-order} standpoint logic (FOSL) and (2) allowing for more modelling flexibility on the side of standpoint descriptions by introducing support for set-theoretical combinations of standpoints (\Cref{sec:fol-standpoint-logic}).
We provide the syntax and semantics of this generic framework, before focusing on the identification of FOSL fragments with beneficial computational properties. To this end, we define the \emph{sentential} fragment, which imposes restrictions on the use of standpoint operators and guarantees a small model property (\Cref{sec:small-model-property}).
Tailored to this case, we introduce a polynomial satisfiability-preserving translation (\Cref{sec:translation}) that does not affect membership in diverse decidable fragments of FO. This allows us to immediately obtain decidability and tight complexity bounds for the standpoint versions of diverse FO fragments (e.g.\ the 2-variable counting fragment, the guarded negation fragment and the triguarded fragment) (\Cref{sec:fragments}).
In addition, it provides a way to leverage off-the-shelf reasoners for practical reasoning in standpoint versions of popular ontology languages. We demonstrate this by extending our results to a standpoint logic based on the description logic $\SROIQbs$, a semantic fragment of FO closely related to the OWL~2~DL ontology language (\Cref{sec:SSSROIQ}).
Finally, we revisit our example to discuss and illustrate properties of our proposal (\Cref{sec:example-medical}).
An extended version of this paper including proofs is available as a {\href{https://arxiv.org/abs/2206.06793}{technical report}}~\cite{gomez-alvarez22arxiv}.


\section{First-Order Standpoint Logic}\label{sec:fol-standpoint-logic}

In this section we introduce the general framework of first order standpoint logic (FOSL) as well as its sentential fragment and establish model theoretic and computational properties.
In addition to establishing various worthwhile decidability and complexity results, this approach also provides us with a clearer view on the underlying principles of our arguments, while avoiding distractions brought about by some peculiarities of expressive ontology languages, which we will address separately later on.


\subsection{FOSL Syntax and Semantics}\label{sec:semantics}

\begin{definition}
	The syntax of first-order standpoint logic ($\FOformulas$) is based on a \define{signature} $\tuple{\Preds, \Consts, \Stands}$, consisting of \define{predicate symbols} $\Preds$ (each associated with an arity \mbox{$n\in\N$}), \define{constant symbols} $\Consts$ and \define{standpoint symbols} $\Stands$, usually denoted with $\sts,\sts'$,
	as well as a set $\Vars$ of \define{variables}, typically denoted with  $x,y,\ldots$ (possibly annotated). These four sets are assumed to be countably infinite and pairwise disjoint.
	The set $\Terms$ of \define{terms} contains all constants and variables, that is, \mbox{$\Terms=\Consts\cup\Vars$}.

	The set $\StandExps$ of \define{standpoint expressions} is defined as follows\/:
	$$\ste_1,\ste_2 \ebnfeq * \mid \sts \mid \ste_1\cup\ste_2 \mid \ste_1\cap\ste_2 \mid \ste_1\setminus\ste_2$$

	The set $\FOformulas$ of \FOSL \define{formulas} is then given by
	$$ \phi,\psi \ebnfeq \pred{P}(t_1,\ldots,t_k) \mid \neg\phi \mid \phi\land\psi \mid \forall x\phi \mid \standbe\phi, $$
	where \mbox{$\pred{P}\in\Preds$} is an $k$-ary predicate symbol, \mbox{$t_1,\dots,t_k\in\Terms$} are terms,
	\mbox{$x\in\Vars$}, and \mbox{$\ste\in\StandExps$}.
\end{definition}
For a formula $\phi$, we denote the set of all of its subformulas by $\Sub(\phi)$.
The \define{size} of a formula is \mbox{$|\phi| \eqdef |\Sub(\phi)|$}.
The connectives and operators $\mathbf{t}$, $\mathbf{f}$, \mbox{$\phi\lor\psi$}, \mbox{$\phi\limplies\psi$}, \mbox{$\exists x\phi$}, and \mbox{$\standd{\ste}\phi$} are introduced as syntactic macros as usual.
As further useful syntactic sugar, we
introduce
\emph{sharpening} statements \mbox{$\ste_1\preceq\ste_2$} to denote \mbox{$\standb{\ste_1\backslash\ste_2}\mathbf{f}$},
the \emph{indeterminacy operator} via \mbox{$\standindef{\ste}\phi := \standd{\ste}\phi \con \standd{\ste}\n\phi$},
and the \emph{determinacy operator} via \mbox{$\standdef{\ste}\phi :=\neg\standindef{\ste}\phi$}.


\begin{definition}\label{def:semantics}
	Given a signature $\tuple{\Preds,\Consts,\Stands}$, a \define{\foss} $\kstruct$ is a tuple $\tuple{\Dom, \Precs, \sigma, \gamma}$ where:
	\begin{itemize}
		\item $\Dom$ is a non-empty set, the \define{domain} of $\kstruct$;
		\item $\Precs$ is the non-empty set of \define{precisifications};
		\item $\sigma$ is a function mapping each standpoint symbol from $\Stands$ to a set of precisifications (i.e., a subset of $\Precs$);
		\item $\gamma$ is a function mapping each precisification from $\Precs$ to an ordinary first-order structure $\struct$ over the domain $\Delta$, whose interpretation function $\intf$ maps\/:
		      \begin{itemize}
			      \item each predicate symbol $\pred{P}{\,\in\,}\Preds$ of arity $k$ to an $k$-ary relation \mbox{$\interprets{\pred{P}} {\,\subseteq\,} \Dom^k$},
			      \item each constant symbol $a{\,\in\,}\Consts$ to a domain element \mbox{$\interprets{a}{\,\in\,}\Dom$}.
		      \end{itemize}
		      Moreover, for any two $\pr_1,\pr_2\in\Precs$ and every $a\in\Consts$ we require $\interpret{a}{\gamma(\pr_1)}=\interpret{a}{\gamma(\pr_2)}$.
	\end{itemize}
\end{definition}
Note that all first-order structures in all precisifications implicitly share the same interpretation domain $\Dom$ given by the overarching \foss $\kstruct$, that is, we adopt the \emph{constant domain assumption}.\footnote{This is not a substantial restriction, as other variants -- expanding domains, varying domains -- can be emulated using constant domains~\cite[Theorem~6]{DBLP:conf/kr/WolterZ98}.}
Moreover, the last condition of \Cref{def:semantics} also enforces \define{rigid constants}, that is, constants denote the same objects in different standpoints (while clearly their properties could differ).
\skipit{
	For any predicate $\pred{P}$ of arity $0$, the empty Cartesian product $\Dom^0$ is simply a one-element set $\{()\}$ (whose only element is the unique zero-tuple $()$ of elements from $\Dom$), so $\interprets{\pred{P}}$ is either $\emptyset$ (interpreted as false) or $\Dom^0$ (interpreted as true), thus adequately recovering the truth value assignments of the propositional atom $\pred{P}$.}

\begin{definition}
	Let $\kstruct=\tuple{\Dom,\Precs,\sigma,\gamma}$ be a \foss for the signature $\tuple{\Preds,\Consts,\Stands}$ and $\Vars$ be a set of variables.
	%
	%
	A \define{variable assignment} is a function $\varassign:\Vars\to\Dom$ mapping variables to domain elements.
	Given a variable assignment $v$, we denote by $\varassign_{\set{x\mapsto\de}}$ the function mapping $x$ to $\de\in\Dom$ and any other variable $y$ to $\varassign(y)$.

	An interpretation function $\intf$ and a variable assignment specify how to interpret terms by domain elements\/:
	\iflong
\else
		We let $\interpret{t}{\struct,\varassign} = \varassign(x)$ if $t=x\in\Vars$, and $\interpret{t}{\struct,\varassign} = \interprets{a}$ if $t=a\in\Consts$.

	\fi
	\iflong
\else
		To interpret standpoint expressions, we lift $\sigma$ from $\Stands$ to all of $\StandExps$ by letting
		$\sigma(*)=\Precs$ and $\sigma(\ste_1\bowtie\ste_2) = \sigma(\ste_1)\bowtie\sigma(\ste_2)$ for ${\bowtie} \in \{\cup, \cap, \setminus\}$.
	\fi

	The satisfaction relation for formulas is defined in the usual way via structural induction.
	In what follows, let $\pr\in\Precs$ and let $\varassign:\Vars\to\Dom$ be a variable assignment;
	we now establish the definition of the satisfaction relation $\models$ for first-order standpoint logic using pointed \fosss:\skipit{\footnote{We implicitly define $\phi$ to be (globally) \emph{satisfiable} iff \mbox{$\kstruct\models\phi$} for some structure $\kstruct$.
			Local satisfiability (satisfiability in some precisification) can be emulated via global satisfiability of $\allstandd\mkern-2mu\phi$.}}%
	\vspace{-1ex}
	\begin{align*}
		 & \kstruct,\!\pr,\!\varassign \models \pred{P}(t_1,\ldots,t_k)\!\! & \text{iff}\ \  & (\interpret{t_1}{\gamma(\pr), \varassign},\ldots,\interpret{t_k}{\gamma(\pr),\varassign}) \in \interpret{\pred{P}}{\gamma(\pr)} \\
		 & \kstruct,\!\pr,\!\varassign \models \neg\phi                     & \text{iff}\ \  & \kstruct,\!\pr,\!\varassign\not\models\phi                                                                                      \\
		 & \kstruct,\!\pr,\!\varassign \models \phi\land\psi                & \text{iff}\ \  & \kstruct,\!\pr,\!\varassign\models\phi \text{ and } \kstruct,\pr,\varassign\models\psi                                          \\
		 & \kstruct,\!\pr,\!\varassign \models \forall x\phi                & \text{iff}\ \  & \kstruct,\!\pr,\!\varassign_{\set{x\mapsto\de}}\models\phi \text{ for all } \de{\,\in\,}\Dom                                    \\
		 & \kstruct,\!\pr,\!\varassign \models \standb{\ste}\phi            & \text{iff}\ \  & \kstruct,\!\pr'\!,\varassign \models \phi \text{ for all } \pr'\in\sigma(\ste)                                                  \\
		 & \kstruct,\!\pr \models \phi                                      & \text{iff}\ \  & \kstruct,\!\pr,\!\varassign\models\phi \text{ for all } \varassign:\Vars\to\Dom                                                 \\
		 & \kstruct \models \phi                                            & \text{iff}\ \  & \kstruct,\!\pr\models\phi \text{ for all } \pr\in\Precs
	\end{align*}
\end{definition}
As usual, $\kstruct$ is a \define{model} for a formula $\phi$ iff \mbox{$\kstruct\models\phi$}.
As an aside, note that the modal-logic nature of \FOSL may become more evident upon realizing that an alternative definition of its semantics via Kripke structures
can be given (with $\standb{\ste}$ interpreted in the standard way) by assigning every \mbox{$\spform{e}\in\StandExps$} the accessibility relation $\{(\pi,\pi') \mid \pi,\pi'\in \Precs, \pi'\in\sigma(e)\}$.

Later in this paper, we will consider cases where the number of precisifications is fixed.
Thus, we conclude this section by a corresponding definition.

\begin{definition}
	For a natural number $n\in \mathbb{N}$, a \FOSL formula $\phi$ is \emph{$n$-satisfiable} iff it has a model $\tuple{\Dom, \Precs, \sigma, \gamma}$ with $|\Pi| = n$.
\end{definition}

\subsection{Small Model Property of Sentential Formulas}\label{sec:small-model-property}

One interesting aspect of standpoint logic is that its simplified Kripke semantics brings about convenient model-theoretic properties that do not hold for arbitrary (multi-)modal logics.
For propositional standpoint logic, it is known that standard reasoning tasks (such as checking satisfiability) are \textsc{NP}-complete~\cite{gomez2021standpoint}, in contrast to \textsc{PSpace}-completeness in related systems such as \textsc{K45}\textsubscript{n}, \textsc{KD45}\textsubscript{n} and \textsc{S5}\textsubscript{n}.
This result is in fact linked to a \define{small model property}, according to which every satisfiable formula has a model with a ``small'' number of precisifications.
This beneficial property only holds in the single-modal \textsc{K45}, \textsc{KD45} and \textsc{S5}~\cite{Pietruszczak2009SimplifiedKD45} but applies to the multi-modal propositional standpoint logic because of its stronger modal interaction. Fortunately, it can also be shown to carry over to some fragments of FOSL and to the use of standpoint expressions. In particular, in this section, we will show that if we restrict the language to those formulas with no free variables in subformulas of the form $\standbe\mkern-1mu\phi$, then we can indeed guarantee that every satisfiable \FOSL formula has a model whose number of precisifications is linear in the size of the formula.

\begin{definition}
	Let $\phi$ be a formula of FOSL.
	We say that $\phi$ is \emph{sentential} iff for all subformulas of $\phi$ of the form $\standbe\mkern-1mu\psi$, all variables occurring in $\psi$ are bound by a quantifier.
\end{definition}

\begin{theorem}
	\label{thm:small-model-property}
	A sentential \FOSL formula $\phi$ is satisfiable iff it has a model 
	with at most $|\phi|$ precisifications. That is, for sentential \FOSL, satisfiability and $|\phi|$-satisfiability coincide.
\end{theorem}

\skipit{
	\begin{proof}
		The ``if'' direction holds trivially.
		For the ``only if'' direction, we consider an arbitrary model $\kstruct=\tuple{\Dom, \Precs, \sigma, \gamma}$ of $\phi$,
		and show that it can be ``pruned'' to obtain a small model $\kstruct'=\tuple{\Dom, \Precsmf{\kstruct}{\phi}, \sigma', \gamma'}$ with $\kstruct'\models\phi$.

		First, we define a set of precisifications that will serve as witnesses for the satisfaction of subformulas preceded by a diamond modality (that is, some $\standbe$ in negative polarity).

		To this end, let $\Precsmf{\kstruct}{\phi}$ be a subset of $\Precs$ with the following property:
		for each subformula $\standb{\ste}\psi\in\Sub(\phi)$ not satisfied in $\kstruct$ (i.e., $\kstruct \not\models\standbe\psi$),
		$\Precsmf{\kstruct}{\phi}$ contains one precisification $\pr\in\Precs$ for which $\pr \in \sigma(\ste)$ and $\kstruct,\pr\not\models\psi$ hold
		(note that the existence of such a $\pr$ follows directly from the definition of $\models$);
		should all $\standb{\ste}\psi\in\Sub(\phi)$ be satisfied in $\kstruct$, we put $\Precsmf{\kstruct}{\phi}=\set{\pr}$ for an arbitrary $\pr\in\Precs$.
		%
		The definition of $\kstruct'$ then concludes by restricting $\sigma$ and $\gamma$ to $\Precsmf{\kstruct}{\phi}$, i.e., we let
		$\sigma'(s)=\sigma(s)\cap\Precsmf{\kstruct}{\phi}$ for each $s\in\Stands$ and $\gamma'(\pr)=\gamma(\pr)$ for each $\pr\in\Precsmf{\kstruct}{\phi}$.
		\iflong
\else
			\noindent An easy induction yields that, for any $\ste\in\StandExps$ and $\pr\in\Precsmf{\kstruct}{\phi}$, we get $\pr{\,\in\,}\sigma(\ste)$ iff $\pr{\,\in\,}\sigma'(\ste)$.
			This also implies $\sigma'(\ste){\,\subseteq\,}\sigma(\ste)$.
		\fi
		\iflong
\else
			With this in mind, another structural induction allows us to conclude $\kstruct'\models\phi$.
		\fi
	\end{proof}}

In the following, it will be convenient to assume that formulas are in \emph{standpoint standard normal form} (SSNF), where no modal operator $\standbe$ occurs inside the scope of another.
Any FOSL formula $\phi$ can be transformed into SSNF in polynomial time.

\skipit{

	\ifsupplementary\else In the following, it will be convenient to assume that formulas are in a normal form that avoids nesting of modalities.\fi

	\iflong
\else
		\begin{definition}[Standpoint Standard Normal Form]
			A first order standpoint logic formula $\phi$ is in \emph{standpoint standard normal form} (SSNF) iff no modal operator $\standbe$ occurs inside the scope of another.
		\end{definition}
	\fi

	\begin{lemma}
		Each first-order standpoint logic formula $\phi$ can be transformed in polynomial time into a formula in SSNF that is a conservative extension of $\phi$.
	\end{lemma}
	\iflong
\else
		Full formal details regarding the translation into SSNF can be found in Appendix A. The transformation follows standard principles. For instance, the formula $\standbe [\exists x\standde[\pred{P}(x)]]$ is transformed into
		\(\standbe[\exists x\,\pred{Q}(x)] \wedge \forall x \big( \pred{Q}(x) \to \standde[\pred{P}(x)]\big)\), where $\pred{Q}$ is a fresh, hitherto unused predicate symbol.
		We will not delve into conservative extensions; for us it is sufficient that this transformation is satisfiability-preserving.
	\fi
	For the remainder of this paper we assume that $\phi$ is in SSNF.
}

\subsection{Translation to Plain First-Order Logic}\label{sec:translation}
\iflong
\else
	In this section, we
	present a translation $\Trans_n$ mapping any \FOSL formula $\phi$ to a plain FO formula $\Trans_n(\phi)$ such that $n$-satisfiability of $\phi$ coincides with satisfiability of $\Trans_n(\phi)$.
	The translation \iflong to be defined\fi
	will make explicit use of a fixed, finite set $\Precsf{n}$ of precisifications with $|\Precsf{n}|=n$.

	Our translation will map any $\phi$ into a formula of (stand\-point-free) first-order logic. The basic idea is to ``emulate'' standpoint structures $\tuple{\Dom, \Precsf{n}, \sigma, \gamma}$ in plain first-order structures over $\Dom$ by means of a ``superposition'' of all $\gamma(\pi)$, which requires to introduce $n$ ``copies'' of the original set of predicates. \skipit{(Constants need not be duplicated as they are rigid, i.e., their interpretation is standpoint-independent.) }
	To this end, we define our first-order vocabulary by $\mathbb{V}_{\!\scaleobj{0.85}{\mathrm{FO}}}(\Preds,\Consts,\Stands,\Precsf{n})=\tuple{\Preds',\Consts}$ where $\Preds'$ contains
	\begin{itemize}
		\item for each predicate \mbox{$\pred{P}\in\Preds$} and precisification \mbox{$\pr\in\Precsf{n}$}, a predicate of the form $\pred{P}_\pr$ of the same arity as $\pred{P}$, intuitively expressing that $\pred{P}_\pr$ should capture $\interpret{\pred{P}}{\gamma(\pr)}$;
		\item for each standpoint constant \mbox{$\sts\in\Stands$} and every precisification \mbox{$\pr\in\Precsf{n}$}, a nullary predicate of the form $\sts_\pr$, intuitively expressing that $\pr\in\sigma(\sts)$.
	\end{itemize}
	The top-level translation is then defined to set\/:
	\begin{equation*}
		\Trans_n(\phi) = \textstyle\bigland_{\pr\in\Precsf{n}}\trans_n(\pr,\phi) \land \textstyle\bigland_{\pr\in\Precsf{n}}{\star}_{\pr},
	\end{equation*}
	where $\trans_n$ is inductively defined by%
	\begin{align*}
		\trans_n(\pr, \pred{P}(t_1,\ldots,t_k)) & = \pred{P}_\pr(t_1,\ldots,t_k)                                                           \\
		\trans_n(\pr, \lnot\psi)                & =\lnot\trans_n(\pr,\psi)                                                                 \\
		\trans_n(\pr, \psi_1{\,\land\,}\psi_2)  & = \trans_n(\pr,\psi_1){\,\land\,}\trans_n(\pr,\psi_2)                                    \\
		\trans_n(\pr, \forall x\psi)            & = \forall x(\trans_n(\pr, \psi))                                                         \\
		\trans_n(\pr'\!, \standb{\ste}\psi)     & = \textstyle \bigland_{\pr\in\Precsf{n}}(\transE(\pr,\ste) \limplies \trans_n(\pr,\psi))
	\end{align*}
\fi
Therein, $\transE$ implements the semantics of standpoint expressions, providing for
each expression \mbox{$\ste\in\StandExps$} a propositional formula $\transE(\pr,\ste)$ over $\set{ s_\pr \guard \pr\in\Precsf{n} }$ as follows\/:%
\begin{align*}
	\transE(\pr,\sts)                  & = \sts_\pr                                        \\
	\transE(\pr,\ste_1\cup\ste_2)      & = \transE(\pr,\ste_1)\lor\transE(\pr,\ste_2)      \\
	\transE(\pr,\ste_1\cap\ste_2)      & = \transE(\pr,\ste_1)\land\transE(\pr,\ste_2)     \\
	\transE(\pr,\ste_1\setminus\ste_2) & = \transE(\pr,\ste_1)\land\neg\transE(\pr,\ste_2)
\end{align*}

%

A routine inspection of the translation ensures that it can be done in polynomial time and its output is of polynomial size, provided it is applied to formulas in SSNF.

\skipit{
	We now show that the translation is correct in the sense mentioned above.
	Consider a translation $\Trans_n(\phi)$ of some stand\-point formula $\phi$.
	Given a “plain” first-order structure $\struct$ of the resulting vocabulary over a domain $\Dom$, the associated \foss \mbox{$\kstruct_\struct=\tuple{\Dom, \Precsf{n}, \sigma, \gamma}$} is clear:
	\begin{itemize}
		\item $\sigma(s)=\set{\pr\in\Precsf{n} \guard \interprets{s_\pr}\neq\emptyset }$ and
		\item $\gamma(\pr)$ is the interpretation function defined by $\pred{P}^{\gamma(\pr)} = \pred{P}_{\!\pr}^\struct$ for $\pred{P} \in \Preds$ and $c^{\gamma(\pr)} = c^\struct$ for $c\in \Consts$.
	\end{itemize}
	For the converse direction, given some
	$\kstruct=\tuple{\Dom, \Precs, \sigma, \gamma}$ with $|\Precs|=n$, we can w.l.o.g.\ assume $\Precs = \Precsf{n}$. We obtain the plain first-order structure ${\struct_\kstruct}$ over the domain $\Delta$ as follows\/:
	\begin{itemize}
		\item For each $s\in\Stands$ and $\pr\in\Precsf{n}$, we set \mbox{$s_{\pr}^{\struct_\kstruct}=\set{()}$} 
		      if \mbox{$\pr\in\sigma(s)$}, and set \mbox{$s_\pr^{\struct_\kstruct}=\emptyset$} otherwise;
		\item for every ``original'' predicate \mbox{$\pred{P}\in\Preds$} and each \mbox{$\pr\in\Precsf{n}$}, we let \mbox{$\pred{P}_{\!\pr}^{\struct_\kstruct} = \pred{P}^{\gamma(\pr)}$};
		      similarly in spirit, for each constant \mbox{$c\in \Consts$} we let \mbox{$c^{\struct_\kstruct} = c^{\gamma(\pr)}$} for an arbitrary $\pi\in \Precsf{n}$.
		\item for every $\pr\in\Precsf{n}$, we set \mbox{$\interprets{{\star}_{\pr}}=\set{()}$}.
	\end{itemize}
	These mappings pave the way for showing soundness and completeness (i.e.\ correctness) of the proposed translation.
	\iflong
\else
		\begin{lemma}
			\label{thm:translation-interpretations}
			Let $\mathbb{V}=\tuple{\Preds,\Consts,\Stands}$ be a standpoint signature and let $\mathbb{V}'=\mathbb{V}_{\!\scaleobj{0.85}{\mathrm{FO}}}(\Preds,\Consts,\Stands,\Precsf{n})$.
			Then, we obtain the following for arbitrary standpoint formulas $\psi \in \FOformulas$, precisifications $\pi \in \Precsf{n}$ and variable assignments $\varassign:\Vars\to\Dom$\/:
			\begin{enumerate}
				\item For any $\mathbb{V}$-standpoint structure \mbox{$\kstruct=\tuple{\Dom, \Precs_{n}, \sigma, \gamma}$} holds\\
				      $\left.\right.\ \kstruct,\pr,\varassign\models\psi \quad\text{if and only if}\quad \struct_\kstruct,\varassign\models\trans_n(\pr,\psi).$
				\item For any $\mathbb{V}'$-interpretation \mbox{$\struct$} holds\\
				      $\kstruct_\struct,\pr,\varassign\models\psi \quad\text{if and only if}\quad \  \ \ \struct,\varassign\models\trans_n(\pr,\psi).\ \ $
			\end{enumerate}
		\end{lemma}
	\fi
	Using this lemma, we can show the announced property.}

\begin{theorem}\label{thm:tanslation-nsat-correct}
	A formula $\phi$ is $n$-satisfiable in \FOSL 
	if and only if
	the formula $\Trans_n(\phi)$ is satisfiable in first-order logic.
\end{theorem}
\skipit{
	\begin{proof}
		\begin{description}[labelwidth=*]
			\item[\normalfont“if” (Soundness):]
			      Let $\struct\models\Trans_n(\phi)$.
			      We use the interpretation $\kstruct_\struct$ defined before and show that $\kstruct_\struct\models\phi$.
			      Firstly, we obtain $\sigma(\star)=\Precsf{n}$ by the second conjunct.
			      Since $\kstruct_\struct\models\phi$ iff $\kstruct_\struct,\pr\models\phi$ for all $\pr\in\Precsf{n}$, we show the latter.
			      Consider any $\pr\in\Precsf{n}$.
			      Clearly $\struct\models\trans_n(\pr,\phi)$ by assumption.
			      By \Cref{thm:translation-interpretations} above, $\kstruct_\struct,\pr\models\phi$.
			\item[\normalfont“only if” (Completeness):]
			      Let $\kstruct=\tuple{\Dom, \Precsf{n}, \sigma, \gamma}$ be such that \mbox{$\kstruct\models\phi$}.
			      We use the first-order structure $\struct_\kstruct$ defined above and show that $\struct_\kstruct\models\Trans_n(\phi)$. 
			      \begin{enumerate}
				      \item $\bigland_{\pr\in\Precsf{n}}\!\trans_n(\pr,\phi)$:
				            By assumption, \mbox{$\kstruct\models\phi$}, that is, \mbox{$\kstruct,\pr\models\phi$} for each \mbox{$\pr\in\Precsf{n}$}.
				            By \Cref{thm:translation-interpretations} above, in each such case we get that \mbox{$\struct_\kstruct\models\trans_n(\pr,\phi)$}.
				            Overall, it follows that \mbox{$\struct_\kstruct\models\bigland_{\pr\in\Precsf{n}}\trans_n(\pr,\phi)$}.
				      \item $\bigland_{\pr\in\Precsf{n}}\!{\star}_{\pr}$: By definition, $\struct_\kstruct\models {\star}_{\pr}$ for any $\pr\in\Precsf{n}$.
			      \end{enumerate}
		\end{description}
	\end{proof}}

In fact, \Cref{thm:tanslation-nsat-correct} provides us with a recipe for a satis\-fia\-bi\-li\-ty-pre\-ser\-ving translation for any formula that comes with a ``small model guarantee'', whenever a bound on the number of precisifications can be computed upfront.
In particular, leveraging \Cref{thm:small-model-property}, we obtain the following corollary.

\begin{corollary}
	A formula $\phi$ is satisfiable in sentential first-order standpoint logic if and only if the formula $\Trans_{|\phi|}(\phi)$ is satisfiable in first-order logic.
\end{corollary}


\section{Expressive Decidable \FOSL Fragments}\label{sec:fragments}

We note that even for the sentential version, first-order standpoint logic is still a generalization of plain first-order logic, whence reasoning in it is undecidable.
Therefore, we will next look into some popular decidable FO fragments and establish decidability and complexity results for reasoning in their sentential standpoint versions.

\begin{definition}
    Let $\mathcal{F}$ denote some FO fragment.
    Then the logic \emph{sentential Standpoint-$\mathcal{F}$}, denoted $\mathbb{S}_{[\mathcal{F}]}$, contains the sentential FOSL formulas $\phi$ where\/:
    \begin{itemize}
        \item all variables inside $\phi$ are bound by some quantifier,
        \item for every subformula $\psi \in \Sub(\mathrm{SSNF}(\phi))$ preceded by a quantifier, $\psi \in \mathcal{F}$ holds.
    \end{itemize}
    Fragment $\mathcal{F}$ is \emph{standpoint-friendly} iff every $\phi \in \mathbb{S}_{[\mathcal{F}]}$ satisfies  $\Trans_{|\phi|}(\mathrm{SSNF}(\phi)) \in \mathcal{F}$.
\end{definition}

\begin{lemma}\label{lem:sfriendly}
    Let $\mathcal{F}$ be a standpoint-friendly fragment of FOL.
    Then the following hold:
    \begin{enumerate}
        \item Satisfiability for $\mathbb{S}_{[\mathcal{F}]}$ is decidable if and only if it is for $\mathcal{F}$.
        \item If the satisfiability problem in $\mathcal{F}$ is at least \textsc{NP}-hard, then the satisfiability problem in $\mathbb{S}_{[\mathcal{F}]}$ is of the same complexity as in $\mathcal{F}$.
    \end{enumerate}
\end{lemma}

It turns out that many popular formalisms are stand\-point-friend\-ly.
For propositional logic (PL), this is straightforward: quantifiers and variables are absent altogether, which is also the reason why sentential Standpoint-PL and proper Standpoint-PL coincide. Thus, \Cref{lem:sfriendly} yields an alternative argument for the NP-completeness of the latter, which was established previously \cite{gomez2021standpoint}.

On the expressive end of the logical spectrum, it is worthwhile to inspect fragments of FO that are still decidable (as a minimal requirement for the feasibility of automated reasoning).
In fact, standpoint-friendliness can be established by structural induction for many of those. Notable examples are\/:
\begin{itemize}
    \item the \emph{counting 2-vari\-able frag\-ment} $\mathrm{C}^2$ \cite{GradelOR97,Pratt-Hartmann05}, which subsumes many \emph{description logics} and serves as a mathematical backbone for related complexity results,
    \item the \emph{guar\-ded negation frag\-ment} $\mathrm{GNFO}$ \cite{GuardedNegation,BaranyBC18}, which encompasses both the popular \emph{guarded fragment} as well as the ubiquitous class of \emph{(unions of) conjunctive queries} also known as \emph{existential positive FO}, and
    \item the tri\-guar\-ded frag\-ment $\mathrm{TGF}$ \cite{RS18,KieronskiR21} (a more recent formalism subsuming both the two-variable and the guarded fragment without equality).
\end{itemize}
Intuitively, standpoint-friendliness for all these (and presumably many more) fragments follows from the fact that they are closed under Boolean combinations of sentences and that the transformation does not affect the structure of quantified formulas.
We therefore immediately obtain that these four popular decidable fragments of FOL allow for accommodating standpoints without any increase in complexity.
\begin{corollary}
    The sentential FOSL fragments $\mathbb{S}_{[\mathrm{PL}]}\ ({=}\,\mathbb{S}_{\mathrm{PL}})$, $\mathbb{S}_{[\mathrm{C}^2]}$, $\mathbb{S}_{[\mathrm{GNFO}]}$, and $\mathbb{S}_{[\mathrm{TGF}]}$ are all decidable and the complexity of their satisfiability problem
    is complete for \textsc{NP}, \textsc{NExpTime}, \textsc{2ExpTime}, and \textsc{N2ExpTime}, respectively.
\end{corollary}
As an aside, we note that all these results remain valid when considering \emph{finite satisfiability} (i.e., restricting to models with finite $\Delta$), because for all considered fragments, companion results for the finite-model case exist and the equisatisfiability argument for our translation preserves (finiteness of) $\Delta$.

\section{Sentential Stand\-point-\texorpdfstring{$\mathcal{SROIQ}b_s$}{SROIQbs}}\label{sec:SSSROIQ}

We next present the highly expressive yet decidable logic \emph{(Sentential) Stand\-point-$\SROIQbs$}, which adds the feature of stand\-point-aware modelling to $\mathcal{SROIQ}b_s$, a description logic (DL) obtained from the well-known DL $\mathcal{SROIQ}$~\cite{DBLP:conf/kr/HorrocksKS06} by a gentle extension of its expressivity, allowing safe Boolean role expressions over simple roles~\cite{DBLP:conf/jelia/RudolphKH08}.\footnote{Focusing on the mildly stronger $\SROIQbs$ instead of the more mainstream $\mathcal{SROIQ}$ allows for a more coherent and economic presentation, without giving up the good computational properties and the availability of optimised algorithms and tools.}
The $\mathcal{SROIQ}$ family serves as the logical foundation of popular ontology languages like OWL~2~DL.
In view of the fact that $\mathcal{SROIQ}b_s$ is a semantic fragment of FO, we can leverage the previously established results and present a sa\-tis\-fi\-a\-bi\-li\-ty-preserving polynomial translation from Stand\-point-$\SROIQbs$ into plain $\SROIQbs$ knowledge bases.
On the theoretical side, this will directly provide us with favourable and tight complexity results for reasoning in Stand\-point-$\SROIQbs$.
On the practical side, this paves the way towards practical reasoning in ``Stand\-point-OWL'', since 
it allows us to use highly optimised OWL~2~DL reasoners off the shelf.
\subsection{\texorpdfstring{$\mathcal{SROIQ}b_s$}{SROIQbs}: Syntax and Semantics}
Let $\Consts$, $\Preds_1$, and $\Preds_2$ be finite, mutually disjoint sets called \emph{individual names}, \emph{concept names} and \emph{role names}, respectively. $\Preds_2$ is subdivided into \emph{simple role names} $\Preds^\mathrm{s}_2$ and \emph{non-simple role names} $\Preds^\mathrm{ns}_2$, the latter containing the \emph{universal role} $\rolU$ and being strictly ordered by some strict order $\prec$.
In the original definition of $\mathcal{SROIQ}b_s$, simplicity of roles and $\prec$ are not given a priori, but meant to be implicitly determined by the set of axioms. Our choice to fix them explicitly upfront simplifies the presentation without restricting expressivity.
Then, the set $\mathcal{R}^\mathrm{s}$ of \emph{simple role expressions} is defined by
\begin{narrowalign}
	$\rolexpR_1,\rolexpR_2 \ebnfeq \rolS \mid \rolS^- \mid \rolexpR_1\cup\rolexpR_2 \mid \rolexpR_1\cap\rolexpR_2 \mid \rolexpR_1\setminus\rolexpR_2$,
\end{narrowalign}
with $\rolS {\,\in\,} \Preds^\mathrm{s}_2$, while the set of (arbitrary) \emph{role expressions} is $\mathcal{R} {\,=\,} \mathcal{R}^\mathrm{s} {\,\cup\,} \Preds^\mathrm{ns}_2$\!. The order $\prec$ is then extended to $\mathcal{R}$ by making all elements of $\mathcal{R}^\mathrm{s}$ $\prec$-minimal.
The syntax of \emph{concept expressions} is given by%
	{\small
		\begin{align*}
			\conC,\conD \ebnfeq \conA \mid \{a\} \mid \top \mid \bot \mid \neg\conC \mid \conC\sqcup\conD \mid \conC\sqcup\conD \mid \forall\rolexpR.\conC \mid \exists\rolexpR.\conC \mid \exists\rolexpR'\!.\mathit{Self} \mid \atmost{n}\rolexpR'\!.C \mid \atleast{n}\rolexpR'\!.C,
		\end{align*}
	}%
with \mbox{$\conA\in \Preds_1$}, \mbox{$a\in \Consts$}, \mbox{$\rolexpR \in  \mathcal{R}$}, \mbox{$\rolexpR' \in  \mathcal{R}^\mathrm{s}$}, and \mbox{$n\in \mathbb{N}$}.
We note that any concept expression can be put in negation normal form, where negation only occurs in front of concept names, nominals, or $\mathit{Self}$ concepts.
The different types of $\mathcal{SROIQ}b_s$ sentences (called \emph{axioms}) are given in Table~\ref{tab:axm}.\footnote{The original definition of \sroiq contained more axioms (role transitivity, (a)symmetry, (ir)reflexivity and disjointness), but these are syntactic sugar in our setting.}

\newcommand{\tuplei}[1]{(#1)}
\noindent\begin{table}[t]
	\setlength{\tabcolsep}{-0.5em}
	\begin{tabular}{ll}
		{
			~\hspace{\stretch{1.5}}
			~\vspace{\stretch{1.5}}
			\!\!\!\!\scalebox{.71}{
				\begin{tabular}[t]{@{}l@{\ \ \,}l@{\ \ \,}l@{}}
					\hline                                                                                                                                                              \\[-2ex]
					Name                & Syntax                            & Semantics                                                                                                 \\\hline
					inverse role        & $\rolS^-$                         & $\{\tuplei{\delta,\delta'}\in\Delta\times\Delta \mid \tuplei{\delta',\delta} \in \rolS^\Inter\}$          \\
					role union          & $\rolexpR_1 \cup \rolexpR_2$      & $\rolexpR_1^\Inter \cup \rolexpR_2^\Inter$                                                                \\
					role intersection   & $\rolexpR_1 \cap \rolexpR_2$      & $\rolexpR_1^\Inter \cup \rolexpR_2^\Inter$                                                                \\
					role difference     & $\rolexpR_1 \setminus \rolexpR_2$ & $\rolexpR_1^\Inter \setminus \rolexpR_2^\Inter$                                                           \\
					universal role      & $\rolU$                           & $\Delta^\Inter\times\Delta^\Inter$                                                                        \\
					\hline                                                                                                                                                              \\[-2ex]			                                                                                  nominal             & $\{a\}$                           & $\{a^{\Inter}\}$                                                                                          \\
					top                 & $\top$                            & $\Delta^\Inter $                                                                                          \\  
					bottom              & $\bot$                            & $\emptyset$                                                                                               \\  
					negation            & $\neg \conC$                      & $\Delta^\Inter \setminus \conC^{\Inter}$                                                                  \\  
					conjunction         & $\conC\sqcap \conD$               & $\conC^{\Inter}\cap \conD^{\Inter}$                                                                       \\  
					disjunction         & $\conC\sqcup \conD$               & $\conC^{\Inter}\cup \conD^{\Inter}$                                                                       \\  
					univ. restriction   & $\forall \rolexpR.\conC$          & $\{\delta \mid \forall y. \tuplei{\delta,\delta'} \in \rolexpR^{\Inter} \to \delta'\in \conC^{\Inter}\}$  \\  
					exist. restriction  & $\exists \rolexpR.\conC$          & $\{\delta \mid \exists y. \tuplei{\delta,\delta'}\in\rolexpR^{\Inter} \wedge \delta'\in \conC^{\Inter}\}$ \\  
					$\Self$ concept     & $\exists\rolexpR.\Self$           & $\{\delta \mid \tuplei{\delta,\delta}\in\rolexpR^{\Inter}\}$                                              \\
					qualified number    & $\atmost{n}\rolexpR.C$            & $\{\delta \mid \#\{\delta'\in \conC^{\Inter}\mid \tuplei{\delta,\delta'} \in \rolexpR^{\Inter}\}\le n\}$  \\
					\qquad restrictions & $\atleast{n}\rolexpR.C$           & $\{\delta \mid \#\{\delta'\in \conC^{\Inter}\mid \tuplei{\delta,\delta'} \in \rolexpR^{\Inter}\}\ge n\}$  \\
					\hline                                                                                                                                                              \\[-2ex]
				\end{tabular}
			}\hspace{\stretch{1.5}}~
		}
		%
		%
		 & {
				~\hspace{\stretch{1.5}}
				~\vspace{\stretch{1.5}}

				\!\scalebox{.71}{
					\begin{tabular}[t]{@{}l@{}r@{\ \ }r@{}}
						\hline                                                                                                                                                                                                                                                       \\[-2ex]
						Name                           & Syntax                                                                                & Semantics                                                                                                                           \\\hline\\[-2ex]
						concept assertion\hspace{-5ex} & $C(a)$                                                                                & $a^\mathcal{I}\in C^\mathcal{I}$                                                                                                    \\
						role assertion                 & $\rolexpR(a,b)$                                                                       & $(a^\mathcal{I},b^\mathcal{I})\in \rolR^\mathcal{I}$                                                                                \\
						equality                       & $a\doteq b$                                                                           & $a^\mathcal{I}=b^\mathcal{I}$                                                                                                       \\
						inequality                     & $a\not\doteq b$                                                                       & $a^\mathcal{I}\neq b^\mathcal{I}$                                                                                                   \\\hline \\[-2ex]
						general concept \hspace{-6ex}  & $C{\,\sqsubseteq\,} D$\!\!                                                            & $C^\mathcal{I}\subseteq D^\mathcal{I}$                                                                                              \\
						~inclusion (GCI) \hspace{-6ex} &                                                                                       &
						\\\hline\\[-2ex]
						role inclusion                 & $\rolexpR_1{\circ}\ldots{\circ}\tad\rolexpR_n{\,\sqsubseteq\,} \rolR$                 & $\rolexpR_1^\mathcal{I}{\circ}\ldots{\circ}\tad\rolexpR_n^\mathcal{I}{\,\subseteq\,} \rolR^\mathcal{I}$                             \\
						~axioms                        & $\rolexpR_1{\circ}\ldots{\circ}\tad\rolexpR_n{\circ}\tad\rolR{\,\sqsubseteq\,} \rolR$ & $\rolexpR_1^\mathcal{I}{\circ}\ldots{\circ}\tad\rolexpR_n^\mathcal{I}{\circ}\tad\rolR^\mathcal{I}{\,\subseteq\,} \rolR^\mathcal{I}$ \\
						~(RIAs)                        & $\rolR{\circ}\tad\rolexpR_1{\circ}\ldots{\circ}\tad\rolexpR_n{\,\sqsubseteq\,} \rolR$ & $\rolR^\mathcal{I}{\circ} \rolexpR_1^\mathcal{I}{\circ}\ldots{\circ}\tad\rolexpR_n^\mathcal{I}{\,\subseteq\,} \rolR^\mathcal{I}$    \\  			                                         & $\rolR{\circ}\tad\rolR{\,\sqsubseteq\,} \rolR$                                     & $\rolR^\mathcal{I}{\circ}\tad\rolR^\mathcal{I}{\,\subseteq\,} \rolR^\mathcal{I}$                                              \\\hline\\[-1ex]
						In RIAs, $\rolR \in \Preds^\mathrm{ns}_2$, while $\rolexpR_i \in \mathcal{R}$ and $\rolexpR_i \prec \rolR$ for all \hspace{-10cm}                                                                                                                            \\
						$i\in \{1,\ldots,n\}$.                                                                                                                                                                                                                                       \\
					\end{tabular}
				}
			}

		\hspace{\stretch{1.5}}~
	\end{tabular}
	\vspace{0.5ex}

	\caption{\small$\mathcal{SROIQ}b_s$ role, concept expressions and axioms. $C\equiv D$ abbreviates $C\sqsubseteq D$, $D\sqsubseteq C$.
	}\label{tab:axm}\label{tab:SROIQ}
	\vspace{-2em}
\end{table}

Similar to FOL, the semantics of $\mathcal{SROIQ}b_s$ is defined via interpretations \mbox{$\mathcal{I}=(\Delta,\cdot^\mathcal{I})$} composed of a non-empty set $\Delta$ called the \emph{domain of $\mathcal{I}$} and a function $\cdot^\mathcal{I}$ mapping individual names to elements of $\Delta$, concept names to subsets of $\Delta$, and role names to subsets of \mbox{$\Delta\times\Delta$}. This mapping is extended to role and concept expressions 
and finally used to define satisfaction of axioms (see Table~\ref{tab:axm}).

\subsection{Standpoint-\texorpdfstring{$\mathcal{SROIQ}b_s$}{SROIQbs}}

The set $\SSS$ of \emph{sentential Stand\-point-$\mathcal{SROIQ}b_s$ sentences} is now defined inductively as follows:
\begin{itemize}
	\item if $\mathbf{Ax}$ is a $\SROIQbs$ axiom then $\mathbf{Ax}\in\SSS$,
	\item if $\phi,\psi\in\SSS$ then $\neg\phi$, as well as $\phi{\,\land\,}\psi$ and $\phi{\,\lor\,}\psi$ are in $\SSS$,
	\item if $\phi\in\SSS$ and $\ste\in\StandExps$ then $\standbe\phi\in\SSS$ and $\standde\phi\in\SSS$.
\end{itemize}

The semantics of sentential Stand\-point-$\mathcal{SROIQ}b_s$ is defined in the obvious way, by ``plugging'' the semantics of $\SROIQbs$ axioms into the semantics of $\SSFO$.
We say a $\SSS$ sentence $\phi$ is in \emph{negation normal form} (NNF), if negation occurs only inside or directly in front of $\mathcal{SROIQ}b_s$ axioms; obviously every Stand\-point-$\mathcal{SROIQ}b_s$ sentence can be efficiently transformed into an equivalent one in NNF.

\subsection{Coping with Peculiarities of \texorpdfstring{$\mathcal{SROIQ}b_s$}{SROIQbs}}
In the following, we will provide a polynomial translation, mapping any $\SSS$ sentence $\phi$ to an equisatisfiable set of $\SROIQbs$ axioms.
This translation is very much in the spirit of the one presented for sentential FOSL, however, $\SROIQbs$ comes with diverse syntactic impediments that we need to circumvent.
Thus, before presenting the translation, we will briefly discuss these issues and how to solve them.

First, $\SROIQbs$ does not provide nullary predicates (i.e., propositional symbols). As a surrogate, we use concept expressions of the form $\forall \rolU.\conA$ which have the pleasant property of holding either for all domain individuals or for none.
Second,
$\SROIQbs$ does not directly allow for arbitrary Boolean combinations of axioms. For all non-RIA axioms, a more or less straightforward equivalent encoding is possible using nominals and the universal role; for instance the expression $\neg[\rolR(a,b)] \vee [\conA \sqsubseteq \conB]$ can be converted into~
$\top \ \sqsubseteq \ \neg\exists \rolU.(\{a\} \sqcap \exists \rolR.\{b\}) \ \sqcup \ \forall \rolU.(\neg\conA \sqcup \conB)$.

Dealing with RIAs requires auxiliary vocabulary; for negated RIAs, we introduce a fresh nominal, say $\{x\}$, to mark the end of a ``violating'' role chain, so $\neg[\rolS \circ \rolS \sqsubseteq \rolR]$ essentially becomes~$\top \ \sqsubseteq \ \exists \rolU.\big((\exists\rolS.\exists\rolS.\{x\}) \sqcap (\neg\exists \rolR.\{x\})\big).$

Unnegated RIAs are even trickier. There is no way of converting them into GCIs, so we have to keep them, but we attach an additional ``guard'', which allows us to disable them whenever necessary. This guard can then be triggered from within a GCI. For an example, consider the expression $[\rolT \circ \rolT' \sqsubseteq \rolR] \vee [\rolT' \circ \rolT \sqsubseteq \rolR]$. Then, introducing fresh ``guard roles'' $\rolS_1$ and $\rolS_2$, we assert the three axioms
$\top \sqsubseteq (\forall \rolU. \exists \rolS_1.\mathit{Self}) \sqcup (\forall \rolU. \exists \rolS_2.\mathit{Self})$ as well as $\rolS_1 \circ \rolT \circ \rolT' \sqsubseteq \rolR$ and $\rolS_2 \circ \rolT' \circ \rolT \sqsubseteq \rolR$.
With this arrangement, the first axiom will ensure that all domain elements carry an $\rolS_1$-loop or all domain elements carry an $\rolS_2$-loop.
Depending on that choice, the corresponding RIA in the second line will behave like its original, unguarded version, while the other one may be entirely disabled.

The introduced strategy for handling positive RIAs has a downside: due to the restricted shapes of RIAs (governed by $\prec$), axioms of the shape $\rolR \circ \rolR \sqsubseteq \rolR$ (expressing transitivity) cannot be endowed with guards. In order to overcome this nuisance, every nonsimple role $\rolR$ has to be accompanied by a subrole
$\underline{\rolR}$, which acts as a ``lower approximation'' of $\rolR$ and -- whenever $\rolR$ is defined transitive -- ``feeds into'' $\rolR$ via tail recursion. This way of reformulating $\rolR \circ \rolR \sqsubseteq \rolR$ allows to attach the wanted guard, but requires adjustments in some axioms that mention $\rolR$.

\subsection{Translation into Plain \texorpdfstring{$\mathcal{SROIQ}b_s$}{SROIQbs}}
We now assume a given $\SSS$ sentence $\phi$, w.l.o.g. in NNF, and provide the formal definition of the translation.
As before, we fix $\Precsf{|\phi|}$ and let our translation's
vocabulary 
$\VSSS(\phi)$ consist of
%
all individual names inside $\phi$, plus, for each \mbox{$\pr \in \Precsf{|\phi|}$}, the following symbols:
(a) a concept name $\conA^\pr$ for each $\conA\in\Preds_1$;
(b) a simple role name $\rolS^\pr$ for each $\rolS{\,\in\,}\Preds^\mathrm{s}_2$;
(c) non-simple role names $\rolR^\pr$ and $\underline{\rolR}^\pr$ for each $\rolR{\,\in\,}\Preds^\mathrm{ns}_2 \!\setminus\! \{\rolU\}$;
(d) a simple role name $\rolS^\pr_\rho$ for each unnegated RIA $\rho$ inside $\phi$;
(e) a fresh constant name $a^\pr_\rho$ for each negated RIA $\rho$ inside $\phi$;
(f) a concept name 
$\conS^\sts_\pr$ for each \mbox{$\sts\in\Stands$}.
Thereby, the non-simple role names inherit their ordering $\prec$ from $\Preds^\mathrm{ns}_2$ and we also let $\underline{\rolR}^\pr \prec \rolR^\pr$ for each $\rolR{\,\in\,}\Preds^\mathrm{ns}_2 \!\setminus\! \{\rolU\}$.

The translation $\Trans(\phi)$ of $\phi$ is then a set of $\mathcal{SROIQ}$ axioms defined as follows:
First, $\Trans(\phi)$ contains the RIA
$\underline{\rolR}^\pr{\,\sqsubseteq\,} \rolR^\pr$
for every $\rolR{\,\in\,}\Preds^\mathrm{ns}_2 \!\setminus\! \{\rolU\}$ and each $\pr\in \Precsf{|\phi|}$.
Second, for every unnegated RIA $\rho$ inside $\phi$ and each \mbox{$\pr\in \Precsf{|\phi|}$}, $\Trans(\phi)$ contains the RIA $BG_\pr(\rho)$, with $BG_\pr$ defined by
	{\small \begin{align*}
			\rolexpR_1{\circ}...{\circ}\tad\rolexpR_n{\,\sqsubseteq\,} \rolR                      \mapsto \hspace{1ex} & \rolS_\rho^\pr{\circ}\tad\rolexpR_1^\pr{\circ}...{\circ}\tad\rolexpR_n^\pr{\,\sqsubseteq\,} \underline{\rolR}^\pr                                     &
			\tad\rolexpR_1{\circ}...{\circ}\tad\rolexpR_n{\circ}\tad\rolR{\,\sqsubseteq\,} \rolR  \mapsto \hspace{1ex} & \rolS_\rho^\pr\tad{\circ}\tad\rolexpR_1^\pr{\circ}...{\circ}\tad\rolexpR_n^\pr{\circ}\tad\underline{\rolR}^\pr{\,\sqsubseteq\,} \underline{\rolR}^\pr   \\
			\rolR\tad{\circ}\tad\rolexpR_1{\circ}...{\circ}\tad\rolexpR_n{\,\sqsubseteq\,} \rolR  \mapsto \hspace{1ex} & \underline{\rolR}^\pr\tad{\circ}\tad\rolexpR_1^\pr{\circ}...{\circ}\tad\rolexpR_n^\pr{\circ}\tad\rolS_\rho^\pr{\,\sqsubseteq\,} \underline{\rolR}^\pr &
			\rolR\tad{\circ}\tad\rolR{\,\sqsubseteq\,} \rolR                                      \mapsto \hspace{1ex} & \rolS_\rho^\pr\tad{\circ}\tad \underline{\rolR}^\pr{\circ}\tad {\rolR}^\pr{\,\sqsubseteq\,}\rolR^\pr,\!\!
		\end{align*}}%
whereby the role expression $\rolexpR^\pr$ is obtained from $\rolexpR$ by substituting every role name $\rolS$ with $\rolS^\pr$ (except $\rolU$, which remains unaltered).
Third and last, $\Trans(\phi)$ contains the GCI
\begin{equation*}
	\top \sqsubseteq \textstyle\bigsqcap_{\pr\in\Precsf{|\phi|}}\trans(\pr,\phi) \sqcap \textstyle\bigsqcap_{\pr\in\Precsf{|\phi|}}\forall \rolU.\conS^{\star}_{\pr}
\end{equation*}
where, by inductive definition,
{\small\begin{align*}
			\trans(\pr\!, \mathbf{Ax})          & = \trans^+(\pr\!,\mathbf{Ax})                                                                    \\
			\trans(\pr\!, \neg\mathbf{Ax})      & = \trans^-(\pr\!,\mathbf{Ax})                                                                    \\
			\trans(\pr\!, \psi_1\land\psi_2)    & = \trans(\pr\!,\psi_1)\sqcap\trans(\pr\!,\psi_2)                                                 \\
			\trans(\pr\!, \psi_1\lor\psi_2)     & = \trans(\pr\!,\psi_1)\sqcup\trans(\pr\!,\psi_2)                                                 \\
			\trans(\pr'\!\!, \standb{\ste}\psi) & = \textstyle\bigsqcap_{\pr\in\Precs_{|\phi|}}(\neg\transE(\pr\!,\ste) \sqcup \trans(\pr\!,\psi)) \\
			\trans(\pr'\!\!, \standd{\ste}\psi) & = \textstyle\bigsqcup_{\pr\in\Precs_{|\phi|}}(\transE(\pr\!,\ste) \sqcap \trans(\pr\!,\psi))
		\end{align*}}
We next present the translation of unnegated and negated $\mathcal{SROIQ}$ axioms ($\rho$ stands for an 
RIA \mbox{$\rolexpR_1{\circ}...{\circ}\tad\rolexpR_m{\,\sqsubseteq\,} \rolR$}):
{\small
\begin{align*}
	\trans^+(\pr\!,\rho)                                                              & = \forall \rolU.\exists \rolS^\pr_{\rho}.\Self                                                                                          &
	\hspace*{-2em}\trans^-(\pr\!,\rho)                  = \exists \rolU.\big((\forall & \rolR^\pr\!\!.\neg \{a_\rho^\pr\}) \sqcap (\exists\underline{\rolexpR}_1^\pr\!...\exists\underline{\rolexpR}_m^\pr.\{a_\rho^\pr\})\big)                                                                                              \\
	\trans^+(\pr\!,C{\,\sqsubseteq\,} D)                                              & = \forall \rolU.(\neg C \sqcup D)^\pr                                                                                                   & \trans^-(\pr\!,C{\,\sqsubseteq\,} D) & = \exists \rolU. (C \sqcap \neg D)^\pr              \\
	\trans^+(\pr\!,C(a))                                                              & = \exists \rolU.\big(\{a\} \sqcap C^\pr\big)                                                                                            & \trans^-(\pr\!,C(a))                 & = \exists \rolU.\big(\{a\} \sqcap (\neg C)^\pr\big) \\
	\trans^+(\pr\!,\rolexpR(a,b))                                                     & = \exists \rolU.\big(\{a\} \sqcap \exists\underline{\rolexpR}^\pr.\{b\}\big)                                                            &
	\trans^-(\pr\!,\rolexpR(a,b))                                                     & = \exists \rolU.\big(\{a\} \sqcap \forall\rolexpR^\pr.\neg\{b\}\big)                                                                                                                                                                 \\
	\trans^+(\pr\!,a \doteq b)                                                        & = \exists \rolU.\big(\{a\} \sqcap \{b\}\big)                                                                                            & \trans^-(\pr\!,a \doteq b)           & = \exists \rolU.\big(\{a\} \sqcap \neg\{b\}\big)
\end{align*}}
Therein, for any role expression $\rolexpR$, we let $\underline{\rolexpR}$ denote $\underline{\rolR}$ if $\rolexpR = \rolR$ is a non-simple role name, and otherwise $\underline{\rolexpR}=\rolexpR$.
Moreover, $\conC^\pr$ denotes the concept expression that is obtained from $\conC$ by transforming it into negation normal form, replacing concept names $\conA$ with $\conA^\pr$ and role expressions $\rolR$ by $\rolR^\pr$, and replacing every $\exists \rolR$ for non-simple $\rolR$ with $\exists \underline{\rolR}$.

As before, $\transE$ implements the semantics of standpoint expressions, but now adjusted to the new framework:
each expression \mbox{$\ste\in\StandExps$} is transformed into a concept expression $\transE(\pr\!,\ste)$ over the 
vocabulary $\set{ \conS^\sts_\pr\!\!\guard\!\! s\in\Stands, \pr{\,\in\,}\Precs_{|\phi|} }$ as follows\/:
{\small \begin{align*}
	\transE(\pr\!,\sts)                  & = \forall \rolU.\conS^\sts_\pr
	\\
	\transE(\pr\!,\ste_1\cup\ste_2)      & = \transE(\pr\!,\ste_1)\sqcup\transE(\pr\!,\ste_2)     \\
	\transE(\pr\!,\ste_1\cap\ste_2)      & = \transE(\pr\!,\ste_1)\sqcap\transE(\pr\!,\ste_2)     \\
	\transE(\pr\!,\ste_1\setminus\ste_2) & = \transE(\pr\!,\ste_1)\sqcap\neg\transE(\pr\!,\ste_2)
\end{align*}}%
With all definitions in place, we obtain the desired result. \skipit{Details can be found in Appendix B.}

\begin{theorem}
	Given $\phi\in\SSS$, the set $\Trans(\phi)$
	\begin{enumerate*}[series=examplecolour,label={(\rm \roman*)},leftmargin=2em,labelwidth=2.4em]
		\item is a valid $\SROIQbs$ knowledge base,
		\item is equisatisfiable with $\phi$,
		\item is of polynomial size wrt.~$\phi$, and
		\item can be computed in polynomial time.
	\end{enumerate*}
\end{theorem}

\section{Example in the Forestry Domain}\label{sec:example-medical}

We consider an extension of \Cref{example:fol} in Sentential Stand\-point-$\SROIQbs$ to illustrate the main reasoning tasks in more detail.
The following additional axiom specifies that
forest land use and urban land use are disjoint subclasses of land (\ref{formula:LandSubclasses_sroiq}).

\begin{enumerate}[resume*=SroiqForestry,label={\rm (F\arabic*)},ref={\rm F\arabic*},leftmargin=2.6em,labelwidth=2.3em]\small
  \item $\allstandb[\pred{ForestlandUse}\dlor\pred{UrbanLandUse}{\,\dlsub\,}\pred{Land}  \ \wedge \
            \pred{ForestlandUse}\dland\pred{UrbanLandUse}{\,\dlsub\,}\bot ]$\label{formula:LandSubclasses_sroiq}
\end{enumerate}

\noindent Now, let us see how, through inferences in $\SSS$,
we can gather unequivocal knowledge (Uneq), obtain knowledge that is relative to a standpoint (Rel), and contrast the knowledge that can be inferred from different standpoints (Cont).
%
For unequivocal knowledge (Uneq), we can infer unambiguously that forests are no urban-use lands:
\begin{narrowalign}
  \mbox{$\allstandb[\pred{Forest}\dlsub\neg\pred{UrbanLandUse}]$}
\end{narrowalign}
\noindent This holds because each precisification must comply with $\mathsf{LC}$ or $\mathsf{LU}$ (\ref{formula:only-union}), and we have $\standbx{LC}[\pred{Forest}\dlsub\neg\pred{UrbanLandUse}]$ from
(\ref{formula:defForestlandCover_sroiq}) and
(\ref{formula:disjointLandEco}), and $\standbx{LU}[\pred{Forest}\dlsub\neg\pred{UrbanLandUse}]$ from (\ref{formula:defForestlandUse_sroiq}) and
(\ref{formula:LandSubclasses_sroiq}).
Regarding relative (Rel) and contrasting (Cont) knowledge,
if we now wanted to query our knowledge base for instances of forest, we would obtain
\vspace{-0.8ex}
{\small \begin{gather*}
    \standbx{LC}[\pred{Forest}(e)]\     \wedge\ \standbx{LC}[\neg\pred{Forest}(l)] \qquad\qquad
    \allstandindef[\pred{Forest}(e)]\   \wedge\ \allstandindef[\pred{Forest}(l)]
  \end{gather*}}
\\[-3.6ex]
The first deduced formula contains knowledge relative to $\mathsf{LC}$, showing its stance on whether the instances constitute a forest, which happens to be conclusive in both cases.
The second formula states the global indeterminacy of both $l$'s and $e$'s membership to the concept $\pred{Forest}$. This stems from the disagreement between the interpretations $\mathsf{LC}$ and $\mathsf{LU}$, whose overall incompatibility ($\standbx{LC\cap LU}[\top{\,\sqsubseteq\,}\bot]$) can also be inferred.

Finally, it is worth looking at the limitations of the sentential fragment of Standpoint $\SROIQbs$. In a non-sentential setting, where modalities can be used at the concept level, ``complex alignments'' or bridges can be established between concepts according to possibly many standpoints. For instance, one can write
\begin{narrowalign}
  \mbox{$\small \standbx{LU}[\pred{Forest}]\dlsub \standbx{LC}[\exists\pred{hasLand}^{-}.\pred{Forest}] \dlor \allstandb[\pred{Cleared}]$} 
\end{narrowalign}
to express that the areas classified as forest according to $\mathsf{LU}$ belong to a forest according to $\mathsf{LC}$ or have been cleared (in which case $\mathsf{LC}$ does not recognise them as forest).
It is an objective of future work of ours to study decidable fragments for which the restrictions on the use of modalities are relaxed to express such kinds of axioms.

\section{Related Work}\label{sec:related-work}


A variety of formal representation systems have been proposed to model perspectives in rather diverse areas of research and with heterogeneous nomenclatures.
Standpoint logic bears some similarities to \emph{context logic} in the style proposed by McCarthy and Buvac \cite{mccarthy1997context}, which has also been applied in a description logic setting \cite{Klarman2016DescriptionContext}. This tradition treats contexts as ``first-class citizens'' of the logic, i.e., full-fledged formal objects over which one can express first-order properties. In contrast, standpoint logic is suitable when a formalisation of the contexts involved is unfeasible, or when the interest resides in the content of the standpoints rather than the context in which they occur.

Another related notion is that of \emph{ontology views}, 
where some works consider potentially conflicting viewpoints \cite{Ribiere2002AOntology,Hemam2011MVP-OWL:Web,Hemam2018Probabilistic}.
Ribière and Dieng~\cite{Ribiere2002AOntology} and Heman et al.~\cite{Hemam2011MVP-OWL:Web,Hemam2018Probabilistic} implement the intuition of \textquote{viewpoints} via ad-hoc extensions of the syntax and semantics of description logics, in a style similar to the work on contextuality by Benslimane et al.~\cite{Benslimane2006ContextualSolutions}. Gorshkov et al.~\cite{Gorshkov2016Multiviewpoint} implement them using named graphs.
Instead, the standpoint approach extends the base language with modalities and provides a Kripke-style semantics for it. This leads to a simpler, more recognisable and more expressive framework that supports, for instance, hierarchies and combinations of standpoints, inferences of partial truths, the preservation of consistency with the established alignments and inferences about the standpoints themselves.  
On a technical level, \fosl can be seen as a many-dimensional (multi-)modal logic~\cite{gabbay03many-dimensional}, whence results from that area apply to our setting.
In particular, the search for non-trivial fragments of first-order modal logics that are still decidable and even practically relevant is an important endeavour, for which we believe that standpoint logic can play useful role.

Finally, in the area of ontology modularity, different formalisms such as \emph{DDL bridge rules} \cite{Borgida2003DistributedSources} and \emph{$\varepsilon$-connections} \cite{kutz2003connections,kutz2004connections} have been proposed to specify the interaction between independent knowledge sources.
These can be related to the present framework in that they provide mechanisms to establish links between conceptual models that do not need to be entirely coherent with each other.
Yet the motivation is inherently different: while the standpoint framework focuses on integrating possibly overlapping knowledge into a global source (while preserving ``standpoint-provenance'' and thus enabling a peaceful coexistence of conflicting information), DDL bridge rules and $\varepsilon$-connections have been devised to establish a certain synchronisation between modules that are and will remain separate.
DDL bridge rules could, however, be simulated within a standpoint framework.

\section{Conclusions and Future Work}\label{sec:conclusions}
The diversity of human world views 
along
with the semantic heterogeneity of natural language are at the heart of well-recognised knowledge interoperability challenges.
As an alternative to the 
common
strategy of merging, we proposed the use of a logical formalism based on the notion of \emph{standpoint} that is suitable for knowledge representation and reasoning with sets of possibly conflicting 
characterisations of a domain.

Using first-order logic as an expressive underlying language, we proposed a multi-modal framework by means of which different agents can establish their individual standpoints (which typically involves specifying constraints and relations), but which also allows for combining standpoints and establishing  alignments between them. Reasoning tasks over such multi-standpoint specifications include gathering unequivocal 
knowledge, determining knowledge that is relative to a standpoint or a set of them, and contrasting the knowledge that can be inferred from different standpoints.

Remarkably, the simplified Kripke semantics allows us to establish a small model property for the \emph{sentential} fragment of \FOSL. This result gave rise to a polynomial, satisfiability-preserving translation into the base logic, which also maintains membership in diverse decidable fragments, immediately implying that for a range of logics, reasoning in their standpoint-enhanced versions does not increase their computational complexity. This indicates that the framework can be applied to ontology alignment, concept negotiation, and knowledge aggregation with inference systems built on top of existing, highly optimised off-the-shelf reasoners.

Future work includes the study of the complexity of FOSL fragments allowing the presence of free variables within the scope of modalities. Note however that in the general case of FOSL, this leads to the loss of the small model property.
\def\btt{\pred{Btt}}
\begin{example}
    \label{exm:no-small-models-monodic}
    Consider the following (non-sentential) FOSL sentence, axiomatising \linebreak "better" ($\btt$) to be interpreted as a non-well-founded strict linear order and requiring for every domain element $x$ (of infinitely many) the existence of some precisification where $x$ is the (one and only) ``best'':
    {\small\begin{align*}
         & \forall xyz \big((\btt(x,y) \land \btt(y,z)) \limplies \btt(x,z) \big)            & \land &
         & \forall xy \neg (\btt(x,y) \land \btt(y,x)) \land                                           \\
         & \forall xy \big( x\neq y \limplies (\btt(x,y) \lor \btt(y,x)) \big)
         & \land                                                                             &
         & \forall x \exists y \btt(x,y) \land \forall x \allstandd\!\neg\exists y \btt(y,x)
    \end{align*}}
    Obviously, this sentence is satisfiable, but only in a model with infinitely many precisifications; that is, the small model property is violated in the worst possible way.
\end{example}
On the other hand, it is desirable from a modelling perspective to allow for some interplay between FO quantifiers and standpoint modalities.
E.g., the non-sentential FOSL sentence
$
    \forall x_1\cdots x_k \big( \pred{P}(x_1,\ldots,x_k) \to \allstandb \pred{P}(x_1,\ldots,x_k) \big)
$
expresses the rigidity of a predicate $\pred{P}$, thereby ``synchronising'' it over all precisifications.

Consequently, we will study how by imposing syntactic restrictions, we can guarantee the existence of small (or at least reasonably-sized) models for non-sentential standpoint formulas.
Results in the field of many-dimensional modal logics~\cite{gabbay03many-dimensional} show that reasoning is decidable for diverse fragments of first-order modal logic such as the \emph{monodic} fragment (where modalities occur only in front of formulas with at most one free variable). However, \Cref{exm:no-small-models-monodic} already shows (by virtue of being within the monodic fragment) that we cannot hope for a small model property even for this slight extension of the sentential fragment.
A detailed analysis of these issues as they apply to the simplified semantics of standpoint logic is the object of current work.

Additionally, we intend to implement the proposed translations and perform experiments to test the performance of the standpoint framework in scenarios of Knowledge Integration. While sentential standpoints can be added at no extra cost in complexity for the discussed fragments in this paper, we intend to run experiments to assess the runtime impact on large knowledge bases with off-the-shelf reasoners.

As another important topic toward the deployment of our framework, we will look into conceptual modelling aspects. Reviewing documented recurrent scenarios and patterns in the area of knowledge integration, we intend to establish guiding principles for conveniently encoding those by using novel strategies possible with structures of standpoints. Examples for such scenarios include the disambiguation of knowledge sources by using combinations of standpoints, and the establishment of bridge-like rules for alignment. For the latter we will investigate their relationship to similar constructs from other frameworks such as $\varepsilon$-connections, distributed description logics, and others.

\ifdl\else\paragraph*{Supplemental Material Statement:} Proofs can be found in the extended version~\cite{gomez-alvarez22arxiv}.\fi

\paragraph*{Acknowledgments:}
Lucía Gómez Álvarez was supported by the \emph{Bundesministerium für Bildung und Forschung} ({BMBF}) in the Center for Scalable Data Analytics and Artificial Intelligence  (\mbox{ScaDS.AI}).
Sebastian Rudolph has received funding from the European Research Council (Grant Agreement no.~771779, DeciGUT).

\bibliographystyle{splncs04}
\bibliography{bib/references-fol,bib/references,bib/mendeley_current}

\ifsupplementary
	\newpage
	\appendix
	\ifdl\else\color{teal}\fi
	\iflong
	\subsection{Standpoint Standard Normal Form }\label{sec:fol-NF}

	In this section we establish the \emph{standpoint standard normal form} and provide a structure-preserving translation based on the translation by Plaisted and Greenbaum~\cite{plaisted1986structure}.
	Normal forms for modal logics have been used to establish completeness, the finite model property and the decidability of some modal systems.
	This normal form will be used subsequently to produce fragment-preserving translations of FOSL into FOL that are polinomial in the size of the input whenever the small model property can be established.

	\begin{definition}[Standpoint Standard Normal Form]

		A first order standpoint logic formula $\phi$ is in \emph{standpoint standard normal form} (SSNF) iff the modal degree of $\phi$ is at most 1, where the degree $\deg(\phi)$ of a formula is defined in the standard way as:

		\begin{align*}
			\deg(P(\svec{t}))       & = 0 \text{ for } P(\svec{t}) \text{ an atomic formula}, \\
			\deg(\neg\psi)          & = \deg(\psi),                                           \\
			\deg(\psi_1\land\psi_2) & = \max\set{\deg(\psi_1),\deg(\psi_2)},                  \\
			\deg(\forall x\psi)     & = \deg(\psi),                                           \\
			\deg(\standbe\psi)      & = \deg(\psi)+1.
		\end{align*}

	\end{definition}

\else
	\section{Translation to SSNF}\label{sec:fol-NF}
\fi

Suppose $\phi$ is a FOSL formula to be translated to SSNF. We provide in what follows a structure-preserving transformation in the spirit of Tseitin transformations.

\begin{definition}[Translation to SSNF] Recall that a \emph{literal} is a formula of the form $P(\svec{t})$ or $\neg P(\svec{t})$, where $P(\svec{t})$ is an atom.
	For any FOSL formula $\psi$, let $L_\psi$ denote a newly introduced literal $P_\psi(x_1,\dots,x_m)$ where $x_1,\dots,x_m$ are the free variables in $\psi$, if $\psi$ is not of the form $\neg\xi$.
	Otherwise, $L_\psi$ is $\neg L_{\xi}$.

	The structure-preserving standard translation of a FOSL formula $\phi$, denoted $SS(\phi)$ is defined by
	induction via $SS(\phi)=\phi$ for each literal $\phi$, and for compound formulas via
	$SS(\phi) = L_\phi \wedge \phi^+$, where



	\begin{itemize}
		\item $(\psi \con \xi)^+ = (L_{\psi \wedge \xi} \imp L_\psi \con L_\xi) \con \psi^+ \con \xi^+$
		\item $(\psi \dis \xi)^+ = (L_{\psi \vee \xi} \imp L_\psi \dis L_\xi) \con \psi^+ \con \xi^+$
		\item $(\neg\psi)^+ = (\psi)^-$
		\item $(\forall x \psi)^+ = (L_{\forall x \psi} \imp \forall x L_{\psi}) \con \psi^+$
		\item $(\exists x \psi)^+ = (L_{\exists x \psi} \imp \exists x L_{\psi}) \con \psi^+$
		\item $(\standbe\psi)^+ = (L_{\standbe \psi} \imp \standbe L_{\psi}) \con \psi^+$
		\item $(\standde \psi)^+ = (L_{\standde \psi} \imp \standde L_{\psi}) \con \psi^+$

	\end{itemize}

	and their homologous negative polarity definitions:

	\begin{itemize}
		\item $(\psi \con \xi)^- = (L_{\psi \wedge \xi} \leftarrow  L_\psi \con L_\xi) \con \psi^- \con \xi^-$
		\item $(\psi \dis \xi)^- = (L_{\psi \vee \xi} \leftarrow L_\psi \dis L_\xi) \con \psi^- \con \xi^-$
		\item $(\neg\psi)^- = (\psi)^+$
		\item $(\forall x \psi)^- = (L_{\forall x \psi} \leftarrow \forall x L_{\psi}) \con \psi^-$
		\item $(\exists x \psi)^- = (L_{\exists x \psi} \leftarrow \exists x L_{\psi}) \con \psi^-$
		\item $(\standbe\psi)^- = (L_{\standbe \psi} \leftarrow \standbe L_{\psi}) \con \psi^-$
		\item $(\standde \psi)^- = (L_{\standde \psi} \leftarrow \standde L_{\psi}) \con \psi^-$
	\end{itemize}


\end{definition}

We now show the correctness of the transformation.

\begin{lemma}\label{Lemma:plaistedT1}

	For all first-order standpoint formulas $\phi$,
	$$ \models (L_\phi \con \phi^+) \imp \phi \text{ and } \models (\neg L_\phi \con \phi^-) \imp \neg \phi .$$

	\begin{proof}

		In this case, we refer to the proof provided by Plaisted and Greenbaum~\cite{plaisted1986structure}. The proof is by induction on the structure of $\phi$, and needs to be extended by showing that the following formulas are valid:

		\begin{itemize}
			\item $L_{\standbe \psi}  \con (\standbe\psi)^+ \imp \standbe \psi$
			\item $\neg L_{\standbe \psi}  \con (\standbe\psi)^- \imp \neg \standbe \psi$

			\item $L_{\standde \psi}  \con (\standde\psi)^+ \imp \standde \psi$
			\item $\neg L_{\standde \psi}  \con (\standde\psi)^- \imp \neg \standde \psi$
		\end{itemize}
		\vspace{1em}

		\begin{description}[labelwidth=*]
			\item[]We first show that $L_{\standbe \psi}  \con (\standbe\psi)^+ \imp \standbe \psi$ is valid.
			      \begin{enumerate}
				      \item By definition $(\standbe\psi)^+$  is $(L_{\standbe \psi} \imp \standbe L_{\psi}) \con \psi^+$.
				      \item By induction hypothesis we know that $\models (L_\psi \con \psi^+) \imp \psi$.
				      \item Let us then assume $L_{\standbe \psi}  \con (\standbe\psi)^+$.
				      \item By (1), we get $(\standbe L_{\psi}) \con \psi^+$.
				      \item Now, note that $\psi^+$ is not embedded by any modal operator, and hence by the semantics, $\psi^+$ holds for all precisifications.
				      \item Then, by (4) and (5), $L_{\psi} \con \psi^+$ holds for all precisifications belonging to $\ste$.
				      \item Thus, by (2), we get $\standbe \psi$ as desired.
			      \end{enumerate}
			      \noindent	We now show that $\neg L_{\standbe \psi}  \con (\standbe\psi)^- \imp \neg \standbe \psi$ is valid.
			      \begin{enumerate}
				      \item By definition $(\standbe\psi)^-$  is $(L_{\standbe \psi} \leftarrow \standbe L_{\psi}) \con \psi^-$.
				      \item By induction hypothesis we know that $\models (\neg L_\psi \con \psi^-) \imp \neg\psi$.
				      \item Let us then assume $\neg L_{\standbe \psi}  \con (\standbe\psi)^-$.
				      \item By (1), we get $(\neg \standbe L_{ \psi}) \con \psi^-$.
				      \item Now, note that $\psi^-$ is not embedded by any modal operator, and hence by the semantics, $\psi^-$  holds for all precisifications.
				      \item Then, by (4) and (5),$\neg L_{\psi} \con \psi^-$ holds for all precisifications belonging to $\ste$.
				      \item Thus, by (2), we get  $\neg \standbe \psi$ as desired.
			      \end{enumerate}
			\item[] The proof for the cases with $\standde \psi$ is analogous to those with $\standbe \psi$.
		\end{description}
	\end{proof}
\end{lemma}

\begin{lemma}\label{Lemma:plaisted1}

	Suppose $\model$ is a model of $\phi$. Let $\model'$ be an interpretation such that for all subformulas $\psi$ of $\phi$,  $\model' \models L_\psi \leftrightarrow \psi$. Suppose $\model'$ interprets any such $\psi$ to be true exactly when $\model$ does.
	Then for all subformulas $\psi$ of $\phi$, $\model'\models \psi^+$  and $\model'\models \psi^-$.

	\begin{proof}
		Again we extend the inductive proof of Plaisted and Greenbaum~\cite{plaisted1986structure}. 

		We consider the case for $\psi = \standbe \xi$ and show $\model'\models \psi^+$, the other cases are analogous. Recall that, by definition, $\psi^+$ is $(L_{\standbe \xi} \imp \standbe L_{\xi}) \con \xi^+$.

		\begin{enumerate}
			\item By induction, $\model' \models \xi^+$.
			\item By assumption, $\model' \models L_{\standbe \xi} \leftrightarrow (\standbe \xi)$,
			\item By assumption also $\model' \models L_{\xi} \leftrightarrow \xi$.
			\item By (2) and (3) we obtain $\model' \models L_{\standbe \xi} \imp \standbe L_\xi$.
			\item Then, together with (1) we get $\model' \models (L_{\standbe \xi} \imp \standbe L_\xi) \con \xi^+$ as desired.
		\end{enumerate}
	\end{proof}
\end{lemma}

\begin{lemma}\label{Lemma:plaisted2}
	If $\phi$ is satisfiable, then so is $L_\phi \con \phi^+$. If $\neg \phi$ is satisfiable, then so is $\neg L_\phi \con \phi^-$.

	\begin{proof}
		Let $\model$ be a model of $\phi$. Let $\model'$ be as in \Cref{Lemma:plaisted1}. Then, $\model'\models \phi^+$ by \Cref{Lemma:plaisted1}. Also, $\model'\models L_\phi$ by the way it is defined. Thus,  $\model'\models L_\phi \con  \phi^+$, which is hence shown to be satisfiable. A similar argument works for  $\model'\models \neg L_\phi \con  \phi^-$.
	\end{proof}
\end{lemma}

\begin{theorem}
	A first-order standpoint formula $\phi$ is satisfiable iff  $L_\phi \con \phi^+$ is satisfiable and  $\neg \phi$ is satisfiable iff $\neg L_\phi \con \phi^-$ is satisfiable.

	\begin{proof}
		Combine \Cref{Lemma:plaistedT1} and \Cref{Lemma:plaisted2}.
	\end{proof}
\end{theorem}

\begin{corollary}
	A first-order standpoint formula $\phi$ is satisfiable iff $SS(\phi)$ is satisfiable.
\end{corollary}

\begin{lemma}
	The translation $SS(\phi)$ of a first-order standpoint formula $\phi$ is in SSNF.

	\begin{proof}
		It is easy to see that the nesting depth of modal operators in $SS(\phi)$ is at most 1, since it is a conjunction, where every conjunct contains at most one such operator.
	\end{proof}
\end{lemma}
\begin{remark}
	Note that if the standpoint formulas are given in NNF, then we can apply the translation
	merely using the positive polarity.
\end{remark}

	\section{\texorpdfstring{$\mathcal{SROIQ}b_s$}{SROIQbs} Translation Equisatisfiability}

Given a “plain” DL interpretation $\struct$ of the resulting vocabulary over a domain $\Dom$, the associated standpoint structure \mbox{$\kstruct_\struct=\tuple{\Dom, \Precsf{n}, \sigma, \gamma}$} is as follows:
\begin{itemize}
	\item	$c^{\gamma(\pr)} = c^\struct$ for $c\in \Consts$.
	\item $\sigma(s)=\set{\pr\in\Precsf{n} \guard \interprets{(\conS^\sts_\pr)}=\Delta }$ and
	\item $\gamma(\pr)$ is the interpretation function defined by
	      \begin{itemize}
		      \item	$\conA^{\gamma(\pr)} = \conA_{\!\pr}^\struct$ for $\conA \in \Preds_1$,
		      \item	$\rolR^{\gamma(\pr)} = \rolR_{\!\pr}^\struct$ for $\rolR \in \Preds_2$,
		      \item  $\rolS^{\gamma(\pr)} = \rolS_{\!\pr}^\struct$ for each $\rolS{\,\in\,}\Preds^\mathrm{s}_2$;
		      \item 	for each $\rolR{\,\in\,}\Preds^\mathrm{ns}_2$, we let $\rolS^{\gamma(\pr)}$ be the transitive closure of $(\underline{\rolR}^\pr)^\struct$  whenever $\rolR\tad{\circ}\tad\rolR{\,\sqsubseteq\,} \rolR$ occurs in $\phi$ and $\exists \rolS^\pr_{\rolR\tad{\circ}\tad\rolR{\,\sqsubseteq\,} \rolR}.\mathit{Self}^\struct = \Delta $; otherwise we let $\rolS^{\gamma(\pr)} = (\underline{\rolR}^\pr)^\struct$.
	      \end{itemize}
\end{itemize}

\medskip

For the converse direction, given some
$\kstruct=\tuple{\Dom, \Precs, \sigma, \gamma}$ with $|\Precs|=n$, we can w.l.o.g.\ assume $\Precs = \Precsf{n}$. We obtain the plain description logic interpretation ${\struct_\kstruct}$ over the domain $\Delta$ as follows\/:
For each constant $c\in \Consts$ we let $c^{\struct_\kstruct} = c^{\gamma(\pr)}$ for an arbitrary $\pr\in \Precsf{n}$. Moreover, for each $\pr\in\Precsf{n}$ we define the following:
\begin{itemize}
	\item for any $s\in\Stands$, we set
	      ${(\conS^\sts_\pr)}^{\struct_\kstruct}=\Delta$ if \mbox{$\pr\in\sigma(s)$}, and set \mbox{$s_\pr^{\struct_\kstruct}=\emptyset$} otherwise;
	\item we set ${(\conS_\pr^{\star})}^{\struct_\kstruct}=\Delta$;
	\item for any ``original'' concept name $\conA \in \Preds_1$, let \mbox{$\conA_{\!\pr}^{\struct_\kstruct} = \conA^{\gamma(\pr)}$}.
	\item for any ``original'' simple role name $\rolS{\,\in\,}\Preds^\mathrm{s}_2$, let \mbox{$\rolS_{\!\pr}^{\struct_\kstruct} = \rolS^{\gamma(\pr)}$};
	\item for every ``original'' non-simple role name $\rolR{\,\in\,}\Preds^\mathrm{ns}_2$, let \mbox{$\underline{\rolR}_{\!\pr}^{\struct_\kstruct} = \rolR_{\!\pr}^{\struct_\kstruct} = \rolR^{\gamma(\pr)}$};
	\item for each unnegated RIA $\rho$ inside $\phi$, let $(\rolS^\pr_\rho)^{\struct_\kstruct} = \{(\delta,\delta) \mid \delta \in \Delta\}$
	      if $\rho$ is satisfied by ${\gamma(\pr)}$ and $(\rolS^\pr_\rho)^{\struct_\kstruct} = \emptyset$ otherwise;
	\item for each negated RIA $\rho = \rolexpR_1{\circ}...{\circ}\tad\rolexpR_m{\,\sqsubseteq\,} \rolR$ inside $\phi$, if $\rho$ is not satisfied by ${\gamma(\pr)}$ pick some witnessing $(\delta,\delta') \in \rolexpR_1^{\gamma(\pr)} {\circ}...{\circ}\tad\rolexpR_m^{\gamma(\pr)} \setminus \rolR^{\gamma(\pr)}$ and let $(a^\pr_\rho)^{\struct_\kstruct} = \delta'$; otherwise (should $\rho$ be satisfied by ${\gamma(\pr)}$) pick $(a^\pr_\rho)^{\struct_\kstruct}$ arbitrarily.
\end{itemize}

With these constructions we can establish the following lemma.

\begin{lemma}
	Let $\mathbb{V}=\tuple{\Preds_1 \cup \Preds_2,\Consts,\Stands}$ be a standpoint signature and let $\mathbb{V}'=\mathbb{V}_{\!\scaleobj{0.85}{\SROIQbs}}(\Preds_1 \cup \Preds_2,\Consts,\Stands,\Precsf{n})$.
	Then, we obtain the following for arbitrary sentential Standpoint-\SROIQbs sentences $\phi \in \SSS$:
	\begin{enumerate}
		\item For any $\mathbb{V}$-standpoint structure \mbox{$\kstruct=\tuple{\Dom, \Precs_{n}, \sigma, \gamma}$} holds
		      \vspace{-1.5ex}
		      $$\ \kstruct\models\phi \quad\text{if and only if}\quad \struct_\kstruct\models\Trans(\phi).$$
		      \vspace{-2ex}
		\item For any $\mathbb{V}'$-interpretation \mbox{$\struct$} holds
		      \vspace{-1ex}
		      $$\kstruct_\struct\models\phi \quad\text{if and only if}\quad \  \ \ \struct\models\Trans(\phi).\ \ $$
	\end{enumerate}
\end{lemma}

From these correspondences, it follows that $\phi$ and $\Trans(\phi)$ are equisatisfiable for every $\phi \in \SSS$

\else
	\ifdl
		\newpage
		\appendix

	\fi
\fi


\end{document}